\documentclass[a4paper,10pt]{article}
\usepackage[a4paper,margin=2.5cm]{geometry}
\usepackage{tabularx}
\usepackage{authblk}
\usepackage{multirow}
\usepackage{subcaption}
\usepackage[labelfont=bf]{caption}
\usepackage{morefloats}
\usepackage{float}
\usepackage{forest}
\usepackage{shuffle}
\usepackage{mathtools}
\usepackage[colorlinks,citecolor=myGreen]{hyperref}
\usepackage{xstring}

\usepackage[]{algorithm2e}

\usepackage{commath} 

\usepackage{amsmath, amsthm, amssymb}
\usepackage{mathabx} 
\usepackage{url}
\usepackage{color}
\usepackage{shuffle}
\usepackage{tikz}
\usetikzlibrary{arrows.meta}

\usepackage[linecolor=white,backgroundcolor=white,bordercolor=white,textsize=tiny]{todonotes}
\usepackage{breqn}

\usepackage{silence}
\usepackage[colorlinks]{hyperref}
\usepackage{cleveref}
\usepackage{graphicx}
\usepackage[font={small,it}]{caption}
\usepackage[flushleft]{threeparttable}
\usepackage{longtable}
\usepackage{enumitem}
\usepackage[normalem]{ulem}

\newtheorem{theorem}{Theorem}
\theoremstyle{plain}

\newtheorem{example}[theorem]{Example}

\theoremstyle{definition} 

\newtheorem{definition}[theorem]{Definition}
\newtheorem*{tata}{Generalization}
  {\begin{mdframed}[backgroundcolor=lightgray]\begin{tata}}%
  {\end{tata}\end{mdframed}}

\usepackage[utf8]{inputenc}

\newcommand{\R}{\mathbb{R}}
\newcommand{\T}{\mathbb{T}}

\newcommand{\N}{\mathbb{N}}
\newcommand{\Z}{\mathbb{Z}}
\newcommand{\E}{\mathbb{E}}
\newcommand{\F}{\mathcal{F}}

\newcommand{\mute}[1]{}

\let\todon\todo
\renewcommand{\todo}[1]{\todon{\color{magenta}#1}}

\newcommand{\evaluatedAt}[1]{\,\raisebox{-.5em}{$\vert_{#1}$}}

\newcommand\DEF[1]{\textbf{#1}}
\renewcommand\O{\mathcal{O}}

\renewcommand\F{\mathcal{F}_d}
\newcommand\CARD{\mathcal{C}}
\renewcommand\N{\mathsf{N}}
\newcommand\NE{\mathsf{NE}}
\renewcommand\E{\mathsf{E}}
\newcommand\SE{\mathsf{SE}}
\renewcommand\S{\mathsf{S}}
\newcommand\SW{\mathsf{SW}}
\newcommand\W{\mathsf{W}}
\newcommand\NW{\mathsf{NW}}

\newcommand\CTS{\mathsf{CTS}}
\newcommand\CTPS{\mathsf{CTPS}}
\newcommand\ALLOWED{\mathsf{Allowed}}
\renewcommand\r{\text{\bf r}}
\newcommand\s{\text{\bf s}}
\renewcommand\t{\text{\bf t}}
\newcommand\CUMSUM{\mathsf{cumsum}}
\newcommand\XX{\mathsf{P}} 

\renewcommand\T{\mathcal{T}} 

\newcommand\zero{\boldsymbol{0}}
\newcommand\one{\boldsymbol{1}}

\newcommand\SEMI{{\mathbb{S}}}
\newcommand\semi{{\scalebox{0.4}{$\mathbb{S}$}}}

\renewcommand\O{\mathcal{O}}
\newcommand{\perm}[1]{%
  \textcolor{black}{%
  \mathtt{%
  [%
    \StrLen{#1}[\stringLength]
    \foreach \n in {1,...,\stringLength}{
      \StrChar{#1}{\n}
      \ifnum\n<\stringLength
      \,
      \fi
    }%
  ]%
  }%
}%
}

\newcommand{\jd}[1]{
\ifvmode\else\newline\fi
{
\fbox{
\parbox{0.8\textwidth}{  \fbox{$\triangleright$\textcolor{magenta}{\textbf{jd}:}} 
#1
}}}\newline}

\newcommand\source{\mathsf{s}}
\newcommand\target{\mathsf{t}}
\definecolor{myGreen}{rgb}{0.18039216 0.49803922 0.09411765}

\begin{document}

\title{Tensor-to-Tensor Models with Fast Iterated Sum Features}

\author[1]{Joscha Diehl\thanks{Corresponding author: \url{joscha.diehl@gmail.com}}}
\author[2]{Rasheed Ibraheem}
\author[3]{Leonard Schmitz}
\author[4]{Yue Wu}

\affil[1]{
    {\small Institute of Mathematics and Computer Science, University of Greifswald,
    Walther-Rathenau-Str. 47,
    Greifswald,
    17489,
    Germany}
}
 
\affil[2]{
    {\small Maxwell Institute for Mathematical Sciences, School of Mathematics, University of Edinburgh,
    The Kings buildings, 
    Edinburgh,
    EH9 3JF, 
    Scotland, UK}
}

\affil[3]{
    {\small Institute of Mathematics, Technical University Berlin,
    Straße des 17. Juni 135,
    Berlin,
    10623,
    Germany}
}

\affil[4]{
    {\small 
    Department of Mathematics and Statistics, University of Strathclyde,
    26 Richmond St,
    Glasgow,
    G1 1XH,
    UK}
}

\maketitle

\begin{abstract}
Data in the form of images or higher-order tensors
is ubiquitous in modern deep learning applications.
Owing to their inherent high dimensionality,
the need for subquadratic layers processing such data
is even more pressing than for sequence data.

We propose a novel tensor-to-tensor layer with linear cost in the input size,
utilizing the mathematical gadget of ``corner trees'' from the field of permutation counting.
In particular, for order-two tensors, we provide an image-to-image layer that
can be plugged into image processing pipelines.
On the one hand, our method can be seen as a higher-order generalization of state-space models.
On the other hand, it is based on a multiparameter generalization of the signature of iterated integrals (or sums).

The proposed tensor-to-tensor concept is used to build a neural network layer called the Fast Iterated Sums (FIS) layer which integrates seamlessly with other layer types. We demonstrate the usability of the FIS layer with both classification and anomaly detection tasks. By replacing some layers of a smaller ResNet architecture with FIS, a similar accuracy (with a difference of only 0.1\%) was achieved in comparison to a larger ResNet while reducing the number of trainable parameters and multi-add operations. The FIS layer was also used to build an anomaly detection model that achieved an average AUROC of 97.3\% on the texture images of the popular MVTec AD dataset. The processing and modelling codes are publicly available at \url{https://github.com/diehlj/fast-iterated-sums}.
    
\end{abstract}

\section{Introduction}

Modern deep learning models for sequential data, especially large language models \cite{sutskever2014sequence,radford2018improving,devlin2019bert}, are built upon the sequence-to-sequence layer. The Transformer layer currently dominates these applications, even with its quadratic complexity \cite{vaswani2017attention,touvron2023llama}. While sequences are order-one tensors, many applications critically rely on tensor-to-tensor layers for higher-order tensors, with image-to-image layers being a prominent example \cite{ronneberger2015u,cciccek20163d,isola2017image,wang2018video}. The typically high dimensionality of such data (e.g., a 1000×1000 RGB image, when flattened, has three million dimensions) makes directly applying the Transformer architecture even more computationally prohibitive than for sequences. Nevertheless, the transformer architecture has been very successfully applied to images,
for example by using a patch-based approach \cite{dosovitskiy2020image}, where an image is divided into smaller, manageable patches that are then processed. 

On a different front, a new wave of research for stream data is zeroing in on linear-cost alternatives. Among these, state-space models (SSMs), which are essentially discretized linear, controlled ordinary differential equations (ODEs) \cite{lyons1994differential},
have emerged as highly promising \cite{gu2021efficiently,gu2023mamba}.
While various SSM formulations have been proposed for image data in earlier works \cite{roesser1975discrete},
their recent application has expanded to image-to-image layers \cite{pauli2024state,liu2024vmamba}.

Mathematically, controlled ODEs and hence state-space models \cite{muca2024theoretical} are closely related
to the \emph{signature
method}, rooted in the theory of \emph{iterated integrals}
\cite{chen1957integration}.
Though not being used for large language models, signature method has proven successful across different
applications in machine learning of time-series data \cite{morrill2021neural,ibraheem2023early,diehlFRUITSFeatureExtraction2024}.

This success is multifaceted, owing to
its universal approximation properties \cite{chevyrev2022signature},
its solid mathematical foundation in control theory, and its adaptability in handling irregularly sampled data while preserving event ordering. Foremost among these advantages, however, is its computational efficiency, based on a dynamic programming principle which yields linear-time computable features.
The method has been successfully integrated with various existing machine learning architectures,
for example Transformer models \cite{moreno2024rough}.

Given the signature method's proven ability to efficiently extract faithful features from one-parameter data, say, sequence data,
it is natural to consider extending its principles to the two-dimensional domain.
The \emph{sums signature} of two-parameter data has been introduced
in \cite{DS23}, which characterizes the data up to a natural equivalence relation.
In contrast to the one-parameter signature, most of its entries are \emph{not} of an 
iterated form, and are not easily computable.
Nevertheless, in that work a subset of (almost) iterated sums is identified, which can be
computed in linear time.  

\emph{Integral}-signatures for images have been proposed in
\cite{diehl2024signature,ZLT_2022,giusti2025topological,lotter2024signaturematricesmembranes},
again with characterizing properties as well as linear-time computability of a subset of the terms. Using terms in the subset that are computable in linear time, some primary
attempts have proven their effectiveness in capturing texure image information
in classification \cite{ZLT22} and anomaly detection \cite{xie20252dsigdetectsemisupervisedframeworkanomaly}.

Using ideas from higher category theory,
a different type of integrals-signature is proposed in
\cite{chevyrev2024multiplicative,lee2024surface}
(see also \cite{lee2023random})
that is in its entirety linear-time computable and easily parallelizable.
Owing to the large equivalence classes of data (images)
on which these higher-categorial integrals-signatures are constant,
it is not clear if
they are expressive enough for practical applications.%
\footnote{We note that this drawback can be mitigated to some extent by ``lifting'' the data to a higher dimensional space,
as is also sometimes done in the one-parameter case.}

In this work, we demonstrate that a substantially larger subset of the sums signature introduced in \cite{DS23} can be computed in linear time than previously believed. Drawing inspiration from permutation pattern counting, we develop a novel differentiable algorithm that efficiently computes this subset with linear time and space complexity relative to the input size. We use this to propose a new deep learning primitive, the \emph{Fast Iterated Sums} (FIS) layer—a tensor-to-tensor operation designed for scalable and expressive representation learning. 

Our contributions can be summarised as follows:
\begin{itemize}
    \item Efficient computation of sums signature (\Cref{sec:fis}):
    We present an algorithm to compute a significant subset of the two-parameter sums signature of \cite{DS23} in linear time and space.
    The approach is based on techniques from permutation pattern counting, specifically using the idea of ``corner trees'' of \cite{even2021counting} (in \Cref{sec:background} we give more precise background needed for the current work).
    This is the first known linear-time method for such a subset of the sums signature, significantly improving on prior approaches with polynomial complexity.
    \item Link to the integrals-signature (\Cref{sec:fis}):
After discretization, our algorithm also yields the first practical method for computing non-diagonal terms of the integrals-signature introduced in \cite{diehl2024signature, giusti2025topological}, also known as the \emph{id-signature}. This represents an important advance in making the full integrals-signature accessible for applications.
    \item Extension to higher-order tensors (\Cref{sec:order_p_tensors}):
The algorithm generalizes to higher-order input tensors, enabling its application in settings beyond matrices (order-two tensors). This includes spatiotemporal or multi-modal data where tensor representations naturally arise.
\item A new deep learning primitive - the FIS Layer (\Cref{sec:FIS_layers_blocks}):
We introduce a novel differentiable tensor-to-tensor layer, the FIS layer, that encapsulates our algorithm. We also define an FIS block by stacking multiple such layers, analogous to stacked convolutional or attention layers in standard architectures. 
\end{itemize}
	We evaluate the performance of the FIS layer on several
	image classification tasks in \Cref{ss:image-classification},
	using the layer as a drop-in replacement
	for convolutional layers in ResNet.
	We further integrate an FIS-backed encoder with a convolution-backed decoder in an auto-encoder network for the anomaly detection tasks.
	We validate the anomaly detector on the well-known texture images of the MVTec AD \cite{bergmann2019mvtec} dataset in \Cref{ss:anomaly-detection}.
	Several experiments were set up to show the robustness of the proposed FIS layer to its hyperparameters and various network architectures, \Cref{sec:appendix}.

\section{Background}
\label{sec:background}

We provide more detail on
three topics mentioned in the introduction:
iterated sums, state space models, and the - seemingly - 
unrelated topic of permutation pattern counting.

\subsection{Iterated sums}
\label{ss:iterated_sums}

We focus here on iterated \emph{sums}, since 
a) when applied to discrete-time data,
   they are a (strict) generalization of iterated integrals \cite{diehl2020time},
b) they form the basis for our discussion of the higher-order generalization,
and
c) they allow to use semirings
different to the usual one of real numbers \cite{diehl2020tropical}.%
\footnote{In particular,
the familiar maxplus semiring will be seen to be beneficial in the experiments.}

For an input sequence $x = (x_1, \dots, x_T)$,
assumed here to be one-dimensional (i.e. scalar-valued) for simplicity,
(algebraic) iterated sums are of the form
\begin{align*}
	y_t = \sum_{1 \le i_1 < \dots < i_k \le t}
			x_{i_1}^{\alpha_1} \dots x_{i_t}^{\alpha_k}, \quad t=1,\dots,T,
\end{align*}
for some positive integers $\alpha_1, \dots, \alpha_k$.
$y_{1:T}$ is computable in linear time using dynamic programming.
Moreover, Blelloch's scan algorithm \cite{blelloch1990prefix} can be used
to saturate available GPU resources.

A learnable sequence-to-sequence layer is obtained by replacing the monomial terms by parameterized functions
(\cite{tothseq2tens,krieg2024}), i.e.
\begin{align*}
	y_t = \sum_{1 \le i_1 < \dots < i_k \le t}
			f^\theta_1(x_{i_1}) \dots f^\theta_k(x_{i_k}), \quad t=1,\dots,T,
\end{align*}
where $f^\theta_1, \dots, f^\theta_k$ are parameterized functions.
Linear-time computation and parallelizability are retained.

The generalization of these methods to images and higher order tensors is challenging,
although several candidates have been proposed.
They can be broadly classified into two subsets.

Geometrically motivated generalizations, using the language of double
categories \cite{chevyrev2024multiplicative,lee2024surface}
have the advantage of being easily parallelizable.
The downside is that they are (much)
less expressive than other methods and are conceptually difficult.

The second subset is based on the observation
that iterated sums are sums over certain \emph{point constellations}
in the (time-)domain of the data. For example, the iterated sum
\begin{align*}
	\sum_{1 \le t_1 < t_2 \le T} x_{t_1} x_{t_2}^2,
\end{align*}
sums over all pairs of integer points $(t_1,t_2)$ in the interval $[1,T]$ which
satisfy $t_1 < t_2$, see \Cref{fig:point_constellations_1d}.

\begin{figure}[H]
	\centering
	\includegraphics[width=0.5\textwidth]{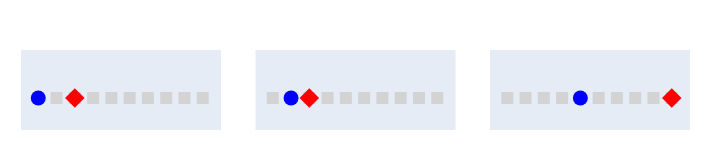}
	\caption{Example point constellations for the iterated sum $\sum\limits_{1 \le {\color{blue}t_1} < {\color{red}t_2} \le T-1} x_{{\color{blue}t_1}} x_{{\color{red}t_2}}^2$. $t_1$ symbolized by a blue circle, $t_2$ by a red diamon.}
	\label{fig:point_constellations_1d}
\end{figure}

This viewpoint can be generalized to two-parameter data,
as is done in \cite{DS23} for arbitrary ``point constellations''
in the plane.
For example the quadruples of
points $\r^1, \r^2, \r^3, \r^4 \in [T_1] \times [T_2]$ in a
rectangular subset of the plane
satisfying the predicate (we write $\r^i = (r^i_1,r^i_2)$)
\begin{align*}
	&\XX(\r^1,\r^2,\r^3,\r^4)
	:=
	          (r^1_1 = r^3_1)
	\bigwedge (r^2_1 = r^4_1)
	\bigwedge (r^1_2 = r^2_2)
	\bigwedge (r^3_2 = r^4_2) \\
	&\qquad\qquad\qquad\qquad
	\bigwedge (r^1_1 < r^2_1)
	\bigwedge (r^3_1 < r^4_1)
	\bigwedge (r^1_2 < r^3_2)
	\bigwedge (r^3_2 < r^4_2),
\end{align*}
lead, for image data $z: [T_1] \times [T_2] \to \R$,
to a sum of the form (see \Cref{fig:point_constellations_2d})
\begin{align}
	\label{eq:thesum}
	\sum_{\substack{\r^1,\r^2,\r^3,\r^4\ :\\ \XX(\r^1,\r^2,\r^3,\r^4)}} z_{\r^1} z_{\r^2} z_{\r^3} z_{\r^4}.
\end{align}

\begin{figure}[H]
	\centering
	\includegraphics[width=0.5\textwidth]{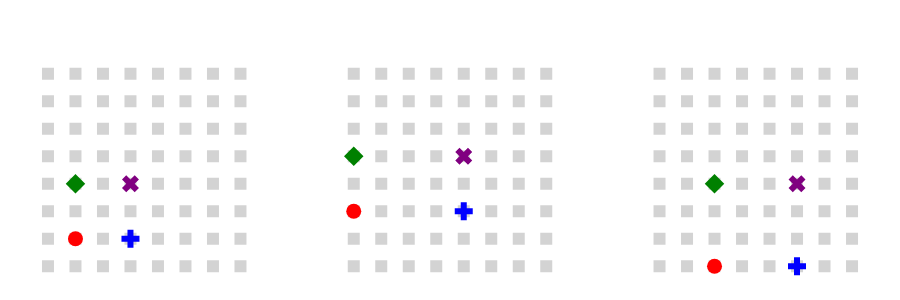}
	\caption{
		Example point constellations for the sum \eqref{eq:thesum}.
		$\r^1$ symbolized  by a red circle,
		$\r^2$ by a blue cross,
		$\r^3$ by a green diamond,
		$\r^4$ by a purple $\mathsf{x}$.
	}
	\label{fig:point_constellations_2d}
\end{figure}

Such expressions are sufficiently expressive (they characterize the input, up to a natural equivalence relation),
but are, in general, not computable in linear time \cite{DS23}.

In the current paper, only a subset of these sums is considered;
namely those whose point constellation can be described by a tree.
The example just given is \emph{not} of this form. The tree structure yields a linear-time algorithm for computing the sum,
which is the main algorithmic core of the current work.

\subsection{State-space models}
\label{ss:ssm}

We explain why state space models do not generalize straightforwardly to
two-parameter data, and why the sums introduced in this paper can nonetheless be
viewed as such a generalization.

A discrete-time affine, time-inhomogeneous
\DEF{state space model}
is of the form
\begin{align*}
	y_{t+1} = A_t y_t + B_t x_t.
\end{align*}
Here $y_t, x_t$ are vector-valued (\DEF{state} and \DEF{input}, respectively),
and $A_t, B_t$ are matrices.
We disregard $B_t$ for now,
and recall that $A$ usually depends on the input $x$.
We then have
\begin{align}
	\label{eq:yT}
	y_T
	=
	A_{T-1}\dots \dots A_1 y_1,
\end{align}
which, if the $A_t$ are noncommuting matrices%
\footnote{A necessary condition to get `interesting' features.},
fundamentally uses the \emph{total order} on the time axis.
It is therefore not obvious how to generalize this to two-parameter data, for example image data.
For example,
\begin{align*}
	\prod_{t_1=1}^{T_1-1}
		\prod_{t_2=1}^{T_2-1}
			A_{t_1,t_2}, 
\end{align*}
depends on the (at this stage, arbitrary)
order of the two products.

Going back to \eqref{eq:yT},
let us consider a
polynomial dependence of $A_t$ on $x_t$
of the form
\begin{align*}
	A_t =
	\begin{bsmallmatrix}
		1 &  x_t^2 & 0 \\
		0 &  1   & x_t \\
		0 &  0   & 1
	\end{bsmallmatrix}
\end{align*}

Then, 
\begin{align*}
	A_{T-1}\dots \dots A_1
	&=
	\begin{bmatrix}
		1 & \sum_{1\le t\le T-1} x_t^2 & \sum_{1 \le t_1 < t_2 \le T-1} x_{t_1} x_{t_2}^2 \\
		0 & 1                      & \sum_{1\le t\le T-1} x_t \\
		0 & 0					  & 1
	\end{bmatrix}.
\end{align*}
We thus see iterated sums (\Cref{ss:iterated_sums}) appearing in the entries of $y_T$.%
\footnote{
Any iterated sum of \cite{diehl2020time} can be obtained
in this fashion for a suitable choice of $A_t$.
Conversely,
since an arbitrary $A_t$ can be approximated by a polynomial (in $x_t$ and $t$),
iterated sums can formally approximate the solution to any state-space model.
}
As we have seen, iterated sums are generalizable to two-parameter data.

\subsection{Permutation patterns and corner trees}
\label{ss:corner_trees}

The counting of occurrences of a (small) permutation
inside a (large) permutation is a fundamental problem in combinatorics,
with applications in computer science and statistics.
If the size of the small permutation is $k$
and the size of the large permutation is $n$,
the naive algorithm has time-complexity $\O(n^k)$.
The breakthrough of \cite{even2021counting} was to provide
an algorithm with linear complexity (up to logarithmic factors)
for a large subspace of (small) permutations.
A \emph{corner tree} is defined in \cite{even2021counting}
as a finite rooted tree with labels from the set of cardinal directions
$\{\NE, \SE, \SW, \NW\}$. In the original work, the labels are
on vertices, but we use the reformulation of \cite{diehl2024generalization}
where the labels are on the edges.
An example of a corner tree is shown in \Cref{fig:corner_tree},
where we also label the vertices for later reference.
\begin{figure}[H]
    \centering
    \begin{subfigure}[b]{0.48\textwidth}
        \centering
        \begin{align*}
            \T =
            \scalebox{0.7}{
            \begin{tikzpicture}[
                node distance=2.7cm,
                line width=1.2pt,
                font=\footnotesize,
                arrows={-Latex},
                baseline={(left.base)}
            ]
                \node[circle,fill=green!20,inner sep=1pt] (root) at (0,0) {$a$};
                \node[circle,fill=red!20,inner sep=1pt] (left) at (-1.5,-2) {$b$};
                \node[circle,fill=blue!20,inner sep=1pt] (right) at (1.5,-2) {$c$};
                \draw (root) -- node[left] {$\SE$} (left);
                \draw (root) -- node[right] {$\NW$} (right);
            \end{tikzpicture}}
        \end{align*}
        \caption{A corner tree.}
        \label{fig:corner_tree}
    \end{subfigure}
    \hfill
    \begin{subfigure}[b]{0.48\textwidth}
        \centering
        \begin{tikzpicture}[scale=0.35]
            \draw[step=1cm,gray!30,very thin] (0,0) grid (5,5);
            
            \draw[->,thick] (0,0) -- (5.5,0) node[right] {};
            \draw[->,thick] (0,0) -- (0,5.5) node[above] {};
            
            \foreach \x in {1,...,5}
                \node[below] at (\x-0.5,0) {\x};
            
            \foreach \y in {1,...,5}
                \node[left] at (0,\y-0.5) {\y};
            
            
            \draw[blue!20,very thick,fill=blue!20] (1.5,4.5) circle (0.25) node[above right, text=black] {$c$}; 
            \draw[green!20,very thick,fill=green!20] (3.5,3.5) circle (0.25) node[above right, text=black] {$a$}; 
            \draw[red!20,very thick,fill=red!20] (4.5,0.5) circle (0.25) node[above right, text=black] {$b$}; 

            \filldraw (0.5,2.5) circle (0.1); 
            \filldraw (1.5,4.5) circle (0.1); 
            \filldraw (2.5,1.5) circle (0.1); 
            \filldraw (3.5,3.5) circle (0.1); 
            \filldraw (4.5,0.5) circle (0.1); 
            
            \draw[<-,black, thick] (1.5,4.5) -- (3.5,3.5) node[midway, below] {\scalebox{0.6}{$\NW$}};
            \draw[->,black, thick] (3.5,3.5) -- (4.5,0.5) node[midway, right] {\scalebox{0.6}{$\SE$}};
        \end{tikzpicture}
        \caption{An occurence in the permutation $\mathtt{[3\,5\,2\,4\,1]}$.}
        \label{fig:corner_tree_permutation}
    \end{subfigure}
    \caption{A corner tree and its occurrence in a permutation.}
    \label{fig:corner_trees_combined}
\end{figure}
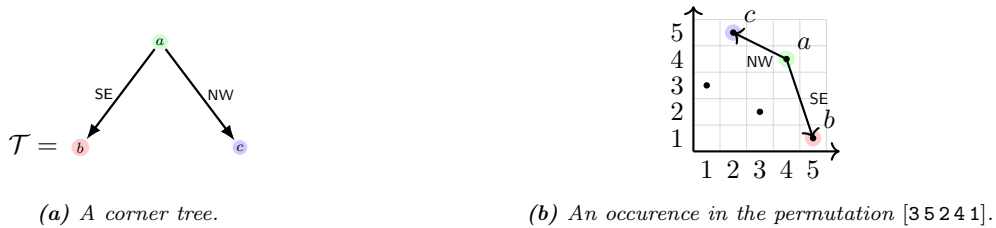

An ``occurence'' of such a corner tree in a permutation is a map
from the vertices of the corner tree to the indices of the permutation such that
the labels of the edges ``are respected''.
More precisely, vertices of the tree are mapped to indices in the permutation
and the east-west direction corresponds to the ordering of the indices
whereas the north-south direction corresponds to the ordering of the values
of the permutation at the indices.

For example, the map $c \mapsto 2$, $a \mapsto 4$, $b \mapsto 5$
is an occurence of the corner tree of \Cref{fig:corner_tree}
in the permutation $\perm{35241}$, visualized in \Cref{fig:corner_tree_permutation}.
A moment's reflection shows that the number of occurences of the corner tree
from \Cref{fig:corner_tree} in facts counts the permutation pattern $\perm{321}$.
In general, corner trees count \emph{linear combinations} of permutation patterns.

Several works have built on this result \cite{beniamini2024counting,diehl2024generalization} to provide efficient (but not necessarily linear)
algorithms to larger subspaces of permutations.
In this work we show how the central idea of corner trees
in \cite{even2021counting} can be used in a seemingly unrelated
context, namely our proposed tensor-to-tensor layer.

\section{Calculating two-parameter iterated sums with corner trees}
\label{sec:fis}

We slightly modify the definition of corner trees
from \cite{even2021counting}.
First, we use the reformulation of \cite{diehl2024generalization}
which puts the direction information on the \emph{edges}.
Moreover, we allow for edge labels from a larger set of eight, instead of four,
cardinal directions.
We furthermore label nodes, indicating arbitrary computation at a single point in the plane (i.e. at a single pixel). 
Denote with $\F$ the set of functions $f: \R^d \to \R$
and let $\CARD := \{\N, \NE, \E, \SE, \S, \SW, \W, \NW\}$.

\begin{definition}
	A \DEF{corner tree} is a rooted, edge- and vertex-labeled
	tree 
	\begin{align*}
		\T = (V(\T), E(\T), \mathfrak v: V(\T) \to \F, \mathfrak e: E(\T) \to \CARD).
	\end{align*}
	Here
	\begin{itemize}
		\item $V(\T)$ is a finite set of vertices,
		\item
			$E(\T) \subset V(\T) \times V(\T)$ is a set of edges,
			such that $(V(\T), E(\T))$ form a rooted tree,%
			\footnote{That is, there is a distinguished vertex called the \DEF{root} such that there is a unique directed path from the root to every vertex.}
		\item $\mathfrak v$ assigns a real-valued function on $\R^d$ to each vertex (this function will be evaluated pointwise when taking sums),
		\item
			$\mathfrak e$ assigns a cardinal direction to each edge (this will be used to determine which points are allowed in the sum,
			see \eqref{eq:allowed} below).
	\end{itemize}

	The edges are considered as directed \emph{away} from the root.
	For an edge $e=(a,b) \in E(\T)$ write $\source(e) = a$
	for the source and $\target(e) = b$ for the target.	
\end{definition}

\medskip

Overloading the notation, we let the cardinal directions
also denote the following logical predicates
on two points $\r = (r_1,r_2), \s = (s_1,s_2) \in \Z^2$
(throughout, $\wedge$ is the logical \emph{and}):
\begin{align*}
	\N( \r, \s )  \ &:=\ \ (r_1 = s_1) \wedge (r_2 < s_2) &\qquad \NE( \r, \s ) \ &:=\ \ (r_1 < s_1) \wedge (r_2 < s_2) \\
	\E( \r, \s ) \  &:=\ \ (r_1 < s_1) \wedge (r_2 = s_2) &\qquad \SE( \r, \s ) \ &:=\ \ (r_1 < s_1) \wedge (r_2 > s_2) \\
	\S( \r, \s )  \ &:=\ \ (r_1 = s_1) \wedge (r_2 > s_2) &\qquad \SW( \r, \s ) \ &:=\ \ (r_1 > s_1) \wedge (r_2 > s_2) \\
	\W( \r, \s )  \ &:=\ \ (r_1 > s_1) \wedge (r_2 = s_2) &\qquad \NW( \r, \s ) \ &:=\ \ (r_1 > s_1) \wedge (r_2 < s_2).
\end{align*}
For example, $\NE(\r,\s)$ is true iff $\s$ is in the northeast quadrant relative to $\r$.

\medskip

Given a corner tree $\T$ with $|V(\T)| = n$ vertices,
write $V(\T) = \{ v_1, \dots, v_n \}$.
For notational simplicity, we assume
that the \emph{root is equal to $v_1$.}

For functions $\r: V(\T) \to \Z^2$
(written as $\r^{v_1}, \dots, \r^{v_n}$ or
$\r^1, \dots, \r^n$)
define
the predicate
\begin{align}
	\label{eq:allowed}
	\ALLOWED(\T, \r)
	:=
	\bigwedge_{e \in E(\T)} \mathfrak e(e)\Big( \r^{\source(e)}, \r^{\target(e)} \Big).
\end{align}

For a function $z: [0,T_1]\times[0,T_2] \to \R^d$ (the data)
we define the \DEF{corner tree sum} of $\T$ as
\begin{align*}
	\CTS(\T,z)
	:=
	\sum_{\substack{\r: V(\T) \to [0,T_1]\times[0,T_2]\ :\\  \ALLOWED(\T,\r)}}
		\prod_{i=1}^n \mathfrak v( v_i )\Big( z_{\r^i} \Big).
\end{align*}

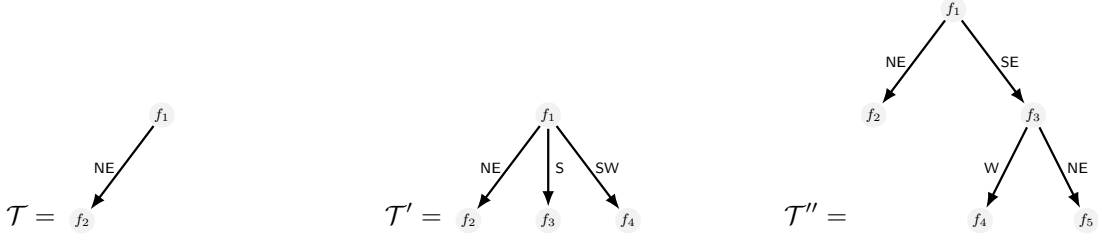
\begin{figure}
    \centering
    \begin{subfigure}[b]{0.30\textwidth}
        \centering
		\begin{align*}
			\T =
			\scalebox{0.7}{
		\begin{tikzpicture}[
			node distance=2.7cm,
			line width=1.2pt,
			font=\footnotesize,
			arrows={-Latex},
			baseline={(left.base)}
		]
			\node[circle,fill=gray!10,inner sep=1pt] (root) at (0,0) {$f_1$};
			\node[circle,fill=gray!10,inner sep=1pt] (left) at (-1.5,-2) {$f_2$};
			\draw (root) -- node[left] {$\NE$} (left);
		\end{tikzpicture}}
		\end{align*}
    \end{subfigure}
    \hfill
    \begin{subfigure}[b]{0.30\textwidth}
        \centering
		\begin{align*}
			\T' = \scalebox{0.7}{
			\begin{tikzpicture}[
			node distance=2.7cm,
			line width=1.2pt,
			font=\footnotesize,
			arrows={-Latex},
			baseline={(left.base)}
			]
			\node[circle,fill=gray!10,inner sep=1pt] (root) at (0,0) {$f_1$};
			\node[circle,fill=gray!10,inner sep=1pt] (left) at (-1.5,-2) {$f_2$};
			\node[circle,fill=gray!10,inner sep=1pt] (bottom) at (0,-2) {$f_3$};
			\node[circle,fill=gray!10,inner sep=1pt] (right) at (1.5,-2) {$f_4$};
			\draw (root) -- node[left] {$\NE$} (left);
			\draw (root) -- node[right] {$\S$} (bottom);
			\draw (root) -- node[right] {$\SW$} (right);
			\end{tikzpicture}
			}
		\end{align*}
    \end{subfigure}
	\hfill
    \begin{subfigure}[b]{0.30\textwidth}
        \centering
		\begin{align*}
			\T'' =
			\scalebox{0.7}{
		\begin{tikzpicture}[
			node distance=2.7cm,
			line width=1.2pt,
			font=\footnotesize,
			arrows={-Latex},
			baseline={(rightleft.base)}
		]
			\node[circle,fill=gray!10,inner sep=1pt] (root) at (0,0) {$f_1$};
			\node[circle,fill=gray!10,inner sep=1pt] (left) at (-1.5,-2) {$f_2$};
			\node[circle,fill=gray!10,inner sep=1pt] (right) at (1.5,-2) {$f_3$};
			\node[circle,fill=gray!10,inner sep=1pt] (rightleft) at (0.5,-4) {$f_4$};
			\node[circle,fill=gray!10,inner sep=1pt] (rightright) at (2.5,-4) {$f_5$};
			\draw (root) -- node[left] {$\NE$} (left);
			\draw (root) -- node[right] {$\SE$} (right);
			\draw (right) -- node[left] {$\W$} (rightleft);
			\draw (right) -- node[right] {$\NE$} (rightright);
		\end{tikzpicture}}
		\end{align*}
    \end{subfigure}
    \caption{Three examples of corner trees.}
    \label{fig:three_corner_trees}
\end{figure}

\begin{example}
	\label{ex:corner_tree}
	With the corner tree $\T$ in \Cref{fig:three_corner_trees} we get
	\begin{align*}
		\ALLOWED(\T, \r)
		&=
		\NE( \r^1, \r^2 )
		=
		(r^1_1 < r^2_1) \wedge (r^1_2 < r^2_2),
	\end{align*}
	and then
	\begin{align*}
		\CTS( \T, z )
		=
		\sum_{
			\substack
			{
				\r^1,\r^2 \in [0,T_1]\times[0,T_2] \\
				(r^1_1 < r^2_1) \wedge
				(r^1_2 < r^2_2)
			}
			}
		f_1( z_{\r^1} ) f_2( z_{\r^2} ).
	\end{align*}
	In words: sum over all points $\r^1,\r^2$
	where $\r^2$ is in the northeast quadrant of $\r^1$,
	evaluating $f_1$ at $z_{\r^1}$ and $f_2$ at $z_{\r^2}$
	and taking the product of these evaluations.
	In the case of a tree of height one,
	we can visualize the allowed point conestellations
	by plotting a point representing the root,
	and the plotting the possible positions corresponding
	to its children.
	For $\T$ this is done in \Cref{fig:order_2_pointconstellation}.

	If $f_1(x) = f_2(x) = x^{(i)}$, for an $i\in [d]$,
	this gives
	\begin{align*}
			\sum_{
			\substack
			{
				\r^1,\r^2 \in [0,T_1]\times[0,T_2]\ : \\
				(r^1_1 < r^2_1) \wedge
				(r^1_2 < r^2_2)
			}
			}
		z^{(i)}_{\r^1} z^{(i)}_{\r^2}.
	\end{align*}
	If
	$z_{r_1,r_2} = x_{r_1+1,r_2+1} - x_{r_1,r_2+1} - x_{r_1+1,r_2} + x_{r_1,r_2}$
	is the discrete approximation of the second derivative $\partial_{12}$ of some $x$,
	we obtain an approximation of the \emph{id-signature} of
	\cite{diehl2024signature} on the word $\mathtt{ii}$.

	\bigskip

	For the corner tree $\T'$ in \Cref{fig:three_corner_trees} we 
	visualize the allowed point constellations in \Cref{fig:order_2_pointconstellation_2}.
\end{example}

\begin{figure}[h]
    \centering
    \begin{subfigure}[b]{0.48\textwidth}
        \centering
        \includegraphics[width=\textwidth]{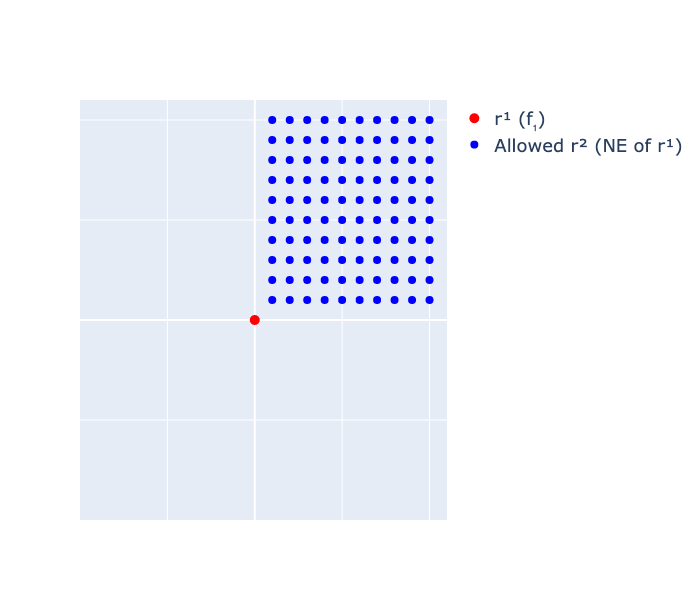}
        \caption{Point constellations for the corner tree $\T$.}
        \label{fig:order_2_pointconstellation}
    \end{subfigure}
    \hfill
    \begin{subfigure}[b]{0.48\textwidth}
        \centering
        \includegraphics[width=\textwidth]{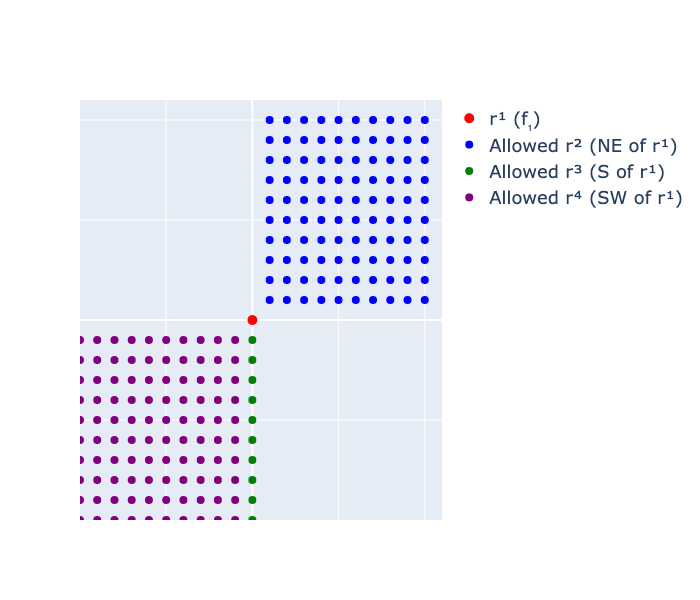}
        \caption{Point constellations for the corner tree $\T'$.}
        \label{fig:order_2_pointconstellation_2}
    \end{subfigure}
    \caption{Visualization of point constellations for different corner trees.}
    \label{fig:point_constellations}
\end{figure}

\begin{example}
	\label{ex:corner_tree_2}
	With the corner tree $\T''$ in \Cref{fig:three_corner_trees} we get
	\begin{align*}
		\CTS( \T'', z )
		=
		\sum_{
			\substack
			{
				(\r^1,\r^2,\r^3,\r^4,\r^5)
                \in ([0,T_1]\times[0,T_2])^5\ : \\
				(r^1_1 < r^2_1) \wedge (r^1_2 < r^2_2),
				(r^1_1 < r^3_1) \wedge (r^1_2 > r^3_2) \\
				(r^3_1 > r^4_1) \wedge (r^3_2 = r^4_2),
				(r^3_1 < r^5_1) \wedge (r^3_2 < r^5_2)
			}
		}
		f_1(z_{\r^1}) f_2(z_{\r^2}) f_3(z_{\r^3}) f_4(z_{\r^4}) f_5(z_{\r^5}).
	\end{align*}
\end{example}

We now derive a recursive formula for $\CTS$
which leads to an efficient algorithm.

For $v \in V(\T)$ denote
the corner subtree rooted at $v$ by $\T\evaluatedAt{v}$. For the rooted tree $\T$ with only one vertex $v_1$,
define the \DEF{corner-tree pre-sum} 
as
\begin{align*}
	\CTPS(\T,z)_{t_1,t_2}
	:= \mathfrak v(v_1)( z_{t_1,t_2} ) \in \R, \quad t_1 \in [0,T_1], t_2 \in [0,T_2].
\end{align*}
Then,
for a tree $\T$ with more than one vertex
define ($t_1 \in [0,T_1], t_2 \in [0,T_2]$)
\begin{align}
	\CTPS(\T,z)_{t_1,t_2}
	:=
	\mathfrak v(v_1)( z_{t_1,t_2} )
	\prod_{ e=(v_1, v_i) \in E(\T) }
		\CUMSUM\left( \mathfrak e(e), \CTPS(\T\evaluatedAt{v_i},z) \right)_{t_1,t_2}.
		\label{eq:CTPS}
\end{align}
Note that the product is over the outgoing edges of the root.
Here, for $\t = (t_1,t_2)$
and $\XX \in \CARD$,
\begin{align}
	\label{eq:cumsum}
	\CUMSUM(\XX, x)_{\t}
	:=
	\sum_{\r : \XX( \t, \r)} x_{\r}.
\end{align}

\begin{example}
	\label{ex:re-expression}
	With the corner tree $\T$ from \Cref{ex:corner_tree},
	\begin{align*}
		\CTPS( \T\evaluatedAt{v_2},z)_{t_1,t_2}
		=
				f_2( z_{t_1,t_2} ),
	\end{align*}
	and
	\begin{align*}
		\CTPS( \T, z)_{t_1,t_2}
		=
		f_1( z_{t_1,t_2} )
		\sum_{
			\substack{
				r_1 > t_1 \\ 
				r_2 > t_2}}
			f_2( z_{r_1,r_2} ).
	\end{align*}

	With the corner tree $\T''$ from
	\Cref{ex:corner_tree_2},
	\begin{align*}
		\CTPS( \T''\evaluatedAt{v_i},z)_{t_1,t_2} &= f_i( z_{t_1,t_2} ),  \quad i=2,4,5.
	\end{align*}
	Then
	\begin{align*}
		\CTPS( \T''\evaluatedAt{v_3},z)_{t_1,t_2}
		&=
		f_3( z_{t_1,t_2} )
		\left( \sum_{\r: \W(\t,\r)} f_4( z_{\r} ) \right)
		\left( \sum_{\r: \NE(\t,\r)} f_5( z_{\r} ) \right),
	\end{align*}
	and
	\begin{align*}
		\CTPS(\T'',z)_{t_1,t_2}
		&=
		f_1( z_{t_1,t_2} )
		\left(\sum_{\r: \NE(\t,\r)} f_2( z_{\r} )\right)
		\left(\sum_{\r: \SE(\t,\r)} \CTPS( \T''\evaluatedAt{v_3},z)_{r_1,r_2}\right).
	\end{align*}

	Note, that in both cases
	\begin{align*}
		\CTS( \T, z )  &= \sum_{(t_1,t_2) \in [0,T_1]\times[0,T_2]} \CTPS( \T, z )_{t_1,t_2} \\
		\CTS( \T'', z ) &= \sum_{(t_1,t_2) \in [0,T_1]\times[0,T_2]} \CTPS( \T'', z )_{t_1,t_2}.
	\end{align*}
\end{example}

We show next that 
the re-expression of the corner tree sum
as in \Cref{ex:re-expression} is possible in general.
\begin{theorem}
	\label{thm:CTS_CTPS}

	For any corner tree $\T$ and data $z$,
	\begin{align*}
            \CTS(\T,z)
		=
		\sum_{\r \in [0,T_1]\times[0,T_2]}
			\CTPS(\T,z)_{\r}.
	\end{align*}
	
\end{theorem}
\begin{proof}
	This follows from
	\begin{align*}
		\ALLOWED(\T, \r)
		=
		\bigwedge_{e = (v_1, v_i) \in E(\T)}
			\mathfrak e(e)\Big( \r^{1}, \r^{i} \Big)
			\wedge
			\ALLOWED(\T\evaluatedAt{v_i}, \r\evaluatedAt{V(\T\evaluatedAt{v_i})}),
	\end{align*}
	where we recall that $v_1$ is the root of $\T$
	and the product is then over the outgoing edges of the root. 
	Here, for a subset $A \subset V(\T)$,
	$\r\evaluatedAt{A}$ denotes the restriction of $\r$ to $A$.
\end{proof}

\subsection{Using different semirings}
\label{subsec:semiring}

We can work over a \DEF{commutative semiring}, \cite{golan2013semirings}.
This is a tuple $(\SEMI,\oplus_\semi,\odot_\semi,\zero_\semi,\one_\semi)$ where 

        $(\SEMI,\oplus_\semi,\zero_\semi)$ is a commutative monoid  with unit $\zero_\semi$,
        $(\SEMI,\odot_\semi,\one_\semi)$ is a commutative monoid with unit $\one_\semi$, 
        zero is attractive, i.e.
        $\zero_\semi \odot_\semi s = s \odot_\semi \zero_\semi = \zero_\semi$ for all $s \in \SEMI$, and
    multiplication distributes over addition.
Assume that the node functions $\mathfrak v$ take values in $\SEMI$ instead of $\R$.
Then the corresponding corner tree sum in the semiring is defined by
replacing in the formulas the usual product by the semiring multiplication $\odot$
and the usual sum by the semiring addition $\oplus$,
\begin{align*}
	\text{\bf CTS}^{\SEMI}(\T,z)
	:=
	\bigoplus_{\r: V(\T) \to [0,T_1]\times[0,T_2]\ :\  \ALLOWED(\T,\r)}
		\bigodot_{i=1}^n \mathfrak v( v_i )\Big( z_{\r^i} \Big).
\end{align*}
Every commutative ring, in particular every field, is a commutative semiring.
The most well-known ``honest'' semiring, also in computer science,
is the max-plus 
semiring $(\R \cup \{-\infty\}, \max, +, -\infty, 0)$.
In this case, the first sum of \Cref{ex:corner_tree} would be
\begin{align*}
	\CTS^{\mathsf{max-plus}}( \T, z )
	=
	\max_{
		\substack
		{
			\r^1,\r^2 \in [0,T_1]\times[0,T_2] \\
			(r^1_1 < r^2_1) \wedge
			(r^1_2 < r^2_2)
		}
		}
	\left( f_1( z_{\r^1} ) + f_2( z_{\r^2} ) \right).
\end{align*}
Semirings for iterated sums of \emph{sequential} data have been considered in
\cite{diehl2020tropical} and in particular the
max-plus semiring has proven useful
in some applications \cite{diehlFRUITSFeatureExtraction2024,krieg2024}.%
\footnote{This comes as no surprise, as 
the popular ReLU networks are precisely the max-plus rational maps
\cite{zhang2018tropical}.
}
We shall see that the max-plus semiring in particular is useful for the corner tree sums of tensors as well.

\subsection{The algorithm}

The recursive formula for $\CTPS$ can be implemented
as an algorithm with space and time complexity $\O(n T_1 T_2)$,
where $n$ is the number of vertices of the corner tree.
Instead of providing pseudo code,
we refer to our annotated
implementation.
We note that since the algorithm centrally uses
cumulative sums, \eqref{eq:cumsum}, it is able to saturate GPU
resources using Blelloch's prefix sum algorithm
(\Cref{ss:iterated_sums}).
Moreover, the calculation of all children
of a node, \eqref{eq:CTPS}, can happen in parallel.

\section{Generalization to order-p tensors}
\label{sec:order_p_tensors}

Mutatis mutandis, the corner tree sums described
in the previous section for order-$2$ tensors
can be applied to order-$p$ tensors.
We describe the case $p=3$, which applies, for example, to video data.

First, to make the generalization easier, we encode the cardinal directions
in the case $p=2$ differently.
\newcommand{\dirpattern}[1]{\texttt{#1}}
We have two axes, north-south and east-west.
In either direction, we can have a positive, zero or negative direction,
which we encode as $\dirpattern{+}, \dirpattern{=}, \dirpattern{-}$.
The labels used so far then correspond to
\footnote{The symbol $\dirpattern{==}$ was unused thus far since it
enforces two points to be at the same position.
This can more easily be encoded
by changing the node label to be a product of functions.}
\begin{align*}
	&\N \leadsto \dirpattern{+=}, \quad
	&&\NE \leadsto \dirpattern{++}, \quad
	&&\E  \leadsto \dirpattern{=+}, \quad
	&&\SE \leadsto \dirpattern{-+}, \\
	&\S  \leadsto \dirpattern{-=}, \quad
	&&\SW \leadsto \dirpattern{--}, \quad
	&&\W  \leadsto \dirpattern{=-}, \quad
	&&\NW \leadsto \dirpattern{+-}.
\end{align*}

Now, for $p=3$, 
the cardinal directions (the edge labels) are
\begin{align*}
	\ell_1\ell_2\ell_3, \quad \ell_1,\ell_2,\ell_3 \in \{\dirpattern{-},\dirpattern{=},\dirpattern{+}\}.	
\end{align*}
An example of a corner tree in this case is
\begin{align}
	\label{eq:corner_tree_order_3}
	\T =
	\scalebox{0.7}{
	\begin{tikzpicture}[
		node distance=2.7cm,
		line width=1.2pt,
		font=\footnotesize,
		arrows={-Latex},
		baseline={(left.base)}
	]
		\node[circle,fill=gray!10,inner sep=1pt] (root) at (0,0) {$f_1$};
		\node[circle,fill=gray!10,inner sep=1pt] (left) at (-1.5,-2) {$f_2$};
		\node[circle,fill=gray!10,inner sep=1pt] (right) at (0,-2) {$f_3$};
		\node[circle,fill=gray!10,inner sep=1pt] (rightright) at (1.5,-2) {$f_4$};
		\draw (root) -- node[left] {$\dirpattern{+++}$} (left);
		\draw (root) -- node[midway, fill=white] {$\dirpattern{=+-}$} (right);
		\draw (root) -- node[right] {$\dirpattern{+==}$} (rightright);
	\end{tikzpicture}}
\end{align}
The corresponding sum is
\begin{align*}
    \sum_{\substack{\r^1,\r^2,\r^3,\r^4 \in [0,T_1]\times[0,T_2]\times[0,T_3] \\
			        r^2_1 \ge r^1_1, 
			        r^2_2 \ge r^1_2, 
			        r^2_3 \ge r^1_3, 
			        r^3_1 = r^1_1, 
			        r^3_2 \ge r^1_2, 
			        r^3_3 \le r^1_3, 
			        r^4_1 \ge r^1_1, 
			        r^4_2 = r^1_2, 
			        r^4_3 = r^1_3, 
					 }}
	f_1(z_{\r^1}) f_2(z_{\r^2}) f_3(z_{\r^3}) f_4(z_{\r^4}).
\end{align*}
The point constellations summed over are visualized in 
\Cref{fig:order_3_pointconstellation}.

\begin{figure}[H]
  \centering
  \includegraphics[width=0.6\textwidth]{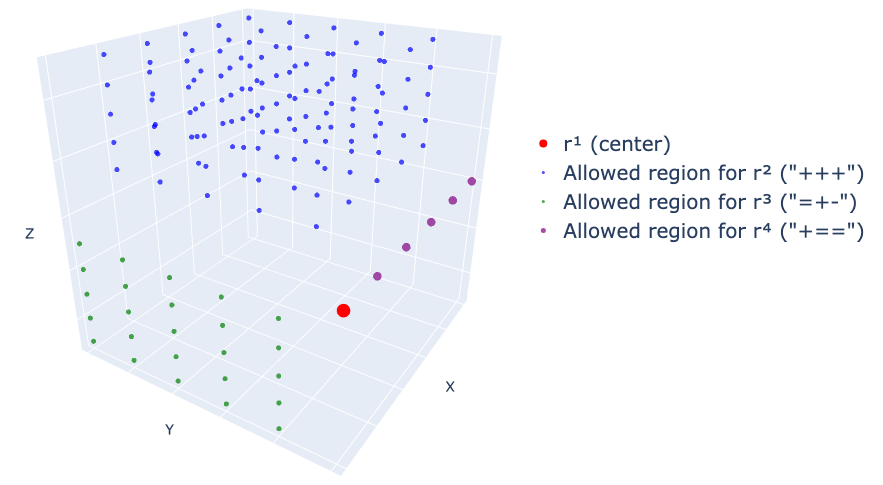}
  \caption{Point constellation for the corner tree in \eqref{eq:corner_tree_order_3}.}
  \label{fig:order_3_pointconstellation}
\end{figure}

The algorithm described in \Cref{sec:fis}
works analogously for this case,
with complexity again being linear in the input size.

\begin{example}

The edge-labels can be thought of as
declaring a ``zone of influence''
of the information, or pixels,
corresponding to a child node
in relation to a parent node.

We visualize this in a three-dimensional space,
where one dimension is thought of as time (think: video data),
and the other two dimensions are thought of as spatial dimensions,
\Cref{fig:tree_and_frames}.
Node $f_1$ is the root node,
and it ``collects'' information from the current frame via its children
$f_2$, $f_3$.
Then, information from prior frames is ``collected''
via the child $f_4$.
This node in turn collects information from the $f_4$'s frame
via its children $f_5$ and $f_6$.

\begin{figure}[H]
	\centering
	\begin{subfigure}[b]{0.48\textwidth}
		\centering
		\scalebox{0.7}{
		\begin{tikzpicture}[
			node distance=2.7cm,
			line width=1.2pt,
			font=\footnotesize,
			arrows={-Latex},
			baseline={(left.base)}
		]
			\node[circle,fill=gray!10,inner sep=1pt] (root) at (0,0) {$f_1$};
			\node[circle,fill=gray!10,inner sep=1pt] (left) at (-1.5,-2) {$f_2$};
			\node[circle,fill=gray!10,inner sep=1pt] (middle) at (0.,-2) {$f_3$};
			\node[circle,fill=gray!10,inner sep=1pt] (right) at (1.5,-2) {$f_4$};
			\node[circle,fill=gray!10,inner sep=1pt] (rightleft) at (0.5,-4) {$f_5$};
			\node[circle,fill=gray!10,inner sep=1pt] (rightright) at (2.5,-4) {$f_6$};
			\draw (root) -- node[left] {$\dirpattern{--=}$} (left);
			\draw (root) -- node[midway, fill=white] {$\dirpattern{+-=}$} (middle);
			\draw (root) -- node[right] {$\dirpattern{==-}$} (right);
			\draw (right) -- node[left] {$\dirpattern{--=}$} (rightleft);
			\draw (right) -- node[right] {$\dirpattern{+-=}$} (rightright);
		\end{tikzpicture}}
		\caption{Corner tree.}
		\label{fig:corner_tree_structure}
	\end{subfigure}
	\hfill
	\begin{subfigure}[b]{0.48\textwidth}
		\centering
		\includegraphics[width=\textwidth]{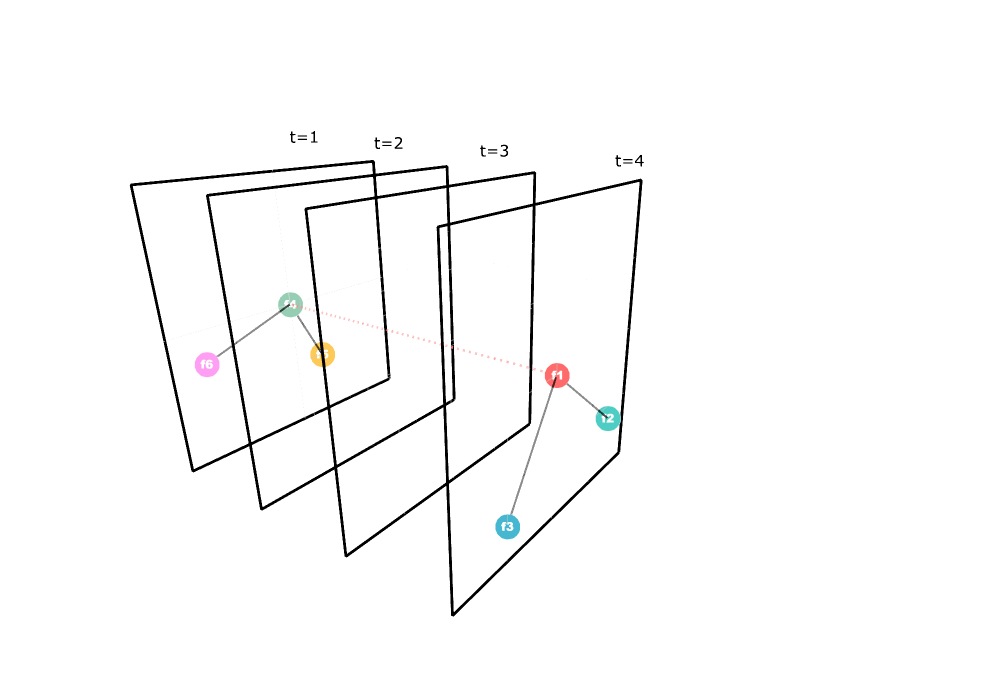}
		\caption{Visualization of a possible point constellation.}
		\label{fig:frames_visualization}
	\end{subfigure}
	\caption{Corner tree and its a three dimensional visualization.}
	\label{fig:tree_and_frames}
\end{figure}
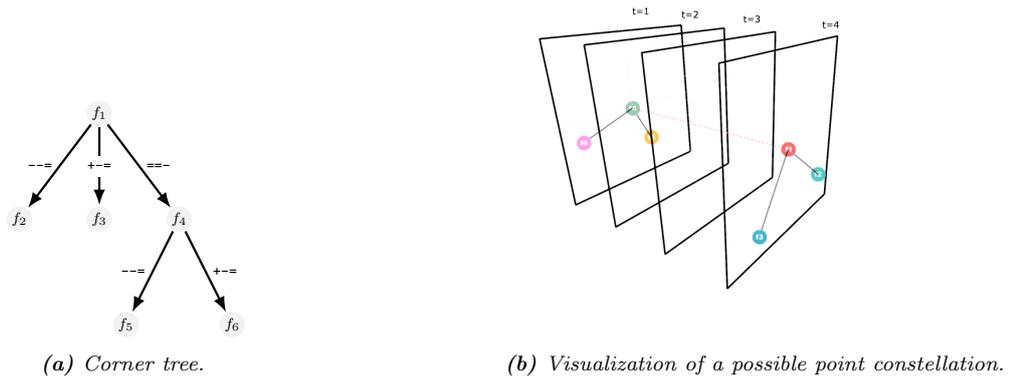

\end{example}

\section{FIS layers and FIS blocks}
\label{sec:FIS_layers_blocks}

This section focuses on the discussion of integrating Fast Iterated Sums (FIS), as described in Section \ref{sec:fis}, into existing neural network architectures to facilitate efficient image data processing while retaining essential structural properties of the visual domain. The fundamental unit of our approach is the FIS layer, whose parameterization and structure we now describe before introducing its integration into a larger building block.

\subsection{The FIS layer}
\label{ss:FIS_layer}

A typical FIS layer (see Figure~\ref{fig:fis-layer-block}(\textbf{a}) for an illustration of its architecture) is hyperparameterized by the number of corner trees, the number of nodes in each tree, the semiring type (real or max-plus), the number of channels in the input image, and a seed (an integer) to ensure reproducibility.
An input of shape $(B, C, H, W)$ to the layer results in an output of shape $(B, N_T, H, W)$, where $B, C, H, W$ and $N_T$ are the number of batches, the number of channels, height, width and the number of corner trees respectively.

The trees used by the layer are generated at random at instantiation, using the seed provided.
The user can choose between unconstrained trees, where only the number of nodes of fixed
($\mathcal T'$ and $\mathcal T''$ in \Cref{fig:three_corner_trees}),
linear trees
(also called ladder trees, or chain trees) where the number of nodes is fixed \emph{and} the tree structure is linear
or
linear-$\NE$ trees, where the number of nodes is fixed, the tree structure is linear, \emph{and} all edges are labeled $\NE$
($\mathcal T$ in \Cref{fig:three_corner_trees}).

The cardinal directions of the edges are drawn, except for the linear-$\NE$ trees, uniformly at instantiation, again using the seed provided.
The functions on the nodes (i.e. $\mathfrak v$ of \Cref{sec:fis})
are taken to be linear, one-dimensional projections of the input.
The weights of these projections constitute the trainable
parameters of this layer.

The FIS layer is proposed to \textit{complement} traditional convolution layers, offering a more expressive and structured feature representation.

\subsection{The FIS block}
\label{ss:FIS_block}

Similar to the Basic Block of a ResNet architecture \cite{he2016deep, chen2020image}, we propose a block of FIS layers and call it FIS Block.
This block starts with two FIS Layers and ends with an Adaptive Pooling Layer (with either average or max pooling).
Each FIS Layer is followed by a Batch Normalization and Rectified Linear Unit Layer (ReLU) for numerical stability and added nonlinearity, respectively.
Figure~\ref{fig:fis-layer-block}(\textbf{b}) illustrates an example of this block. An input of shape $(B, C, H, W)$ results in an output of shape $(B, N_T, H', W')$, where $(H', W')$ is the output size of the chosen Adaptive Pooling. 

\begin{figure}[H]
    \centering
    \begin{subfigure}[h]{0.49\textwidth}
	    \centering
        \includegraphics[width=0.6\textwidth]{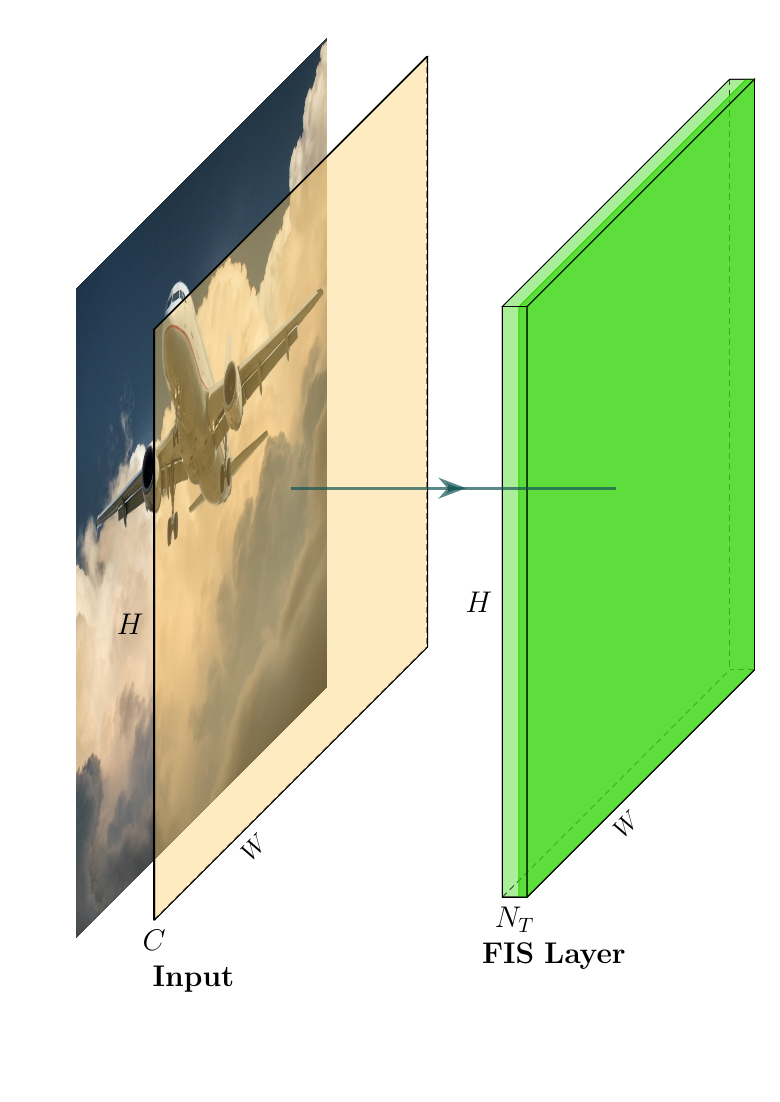}
         \caption{}
     \end{subfigure}	
     \begin{subfigure}[h]{0.49\textwidth}
	    \centering
        \includegraphics[width=0.6\textwidth]{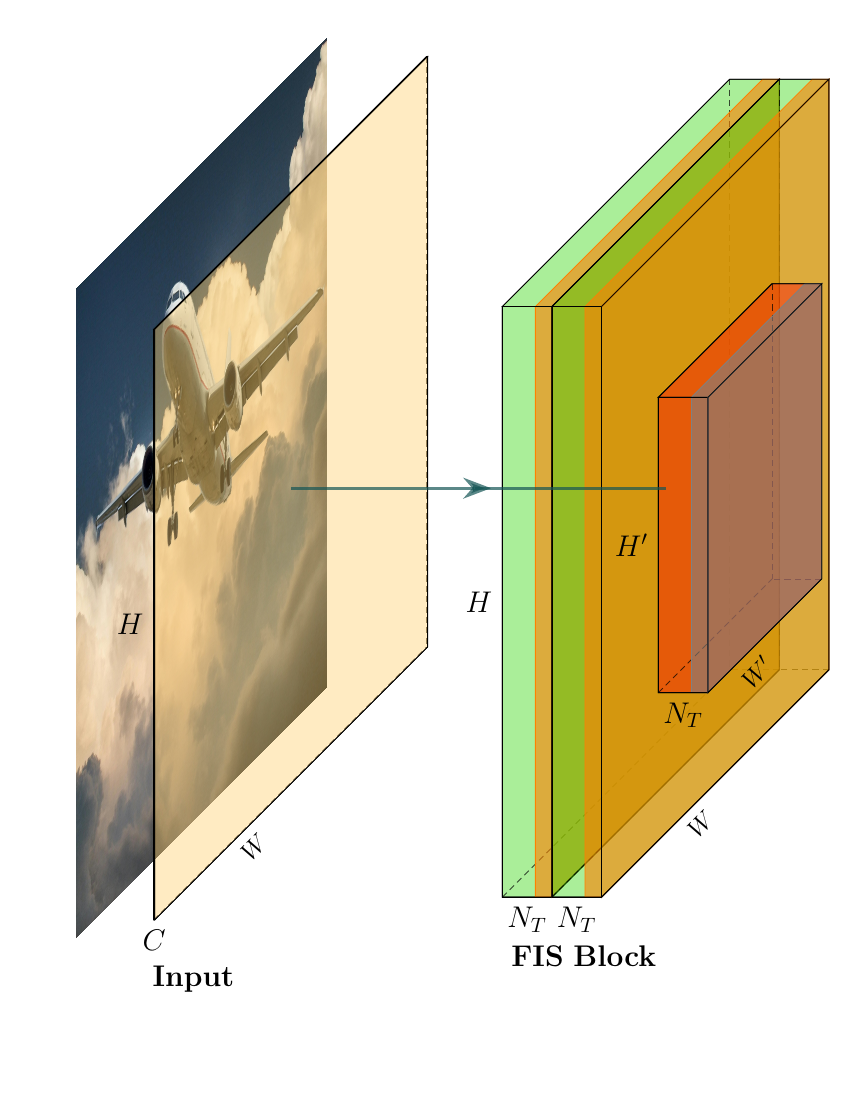}
        \caption{}
   \end{subfigure}	
   
 \caption{Illustraction of (\textbf{a}) FIS Layer and (\textbf{b}) FIS Block architecture. An FIS Block comprises two FIS Layers (in green) and an Adaptive Pooling Layer (in red) of output size $(H', W')$. A Batch Normalization and RELU layers (both in yellow) follow each FIS Layer.}
\label{fig:fis-layer-block}
\end{figure}

\section{Applications}
In this section, we evaluate the performance of the Fast Iterated Sums (FIS) features on two canonical tasks in computer vision: image classification and anomaly detection. All the experiments were implemented in PyTorch and were run on one NVIDIA RTX A4000 GPU.

\subsection{Image classification}
\label{ss:image-classification}

We consider image classification as our first application to showcase the effectiveness of our proposed feature representation in image-related tasks, as it is a fundamental problem in computer vision with well-established evaluation benchmarks. 
To ensure a fair and meaningful comparison, we choose ResNet \cite{chen2020image} as our benchmark due to its widespread adoption, strong performance, and well-understood architectural properties. Specifically, we evaluate our approach on CIFAR-10 and CIFAR-100 \cite{krizhevsky2009learning}, two widely used datasets that cover a diverse range of image classification challenges, allowing us to assess the generalizability of our method across different levels of task complexity.
ResNet's established, though not state-of-the-art, baseline enables clear assessment of our feature
representation's impact.
We focus on CIFAR datasets to demonstrate theoretical validity before scaling to larger datasets like ImageNet.

\subsubsection{Classifier network architecture}
\label{sec:classification-intro}

To showcase the applicability of the FIS features to a classification problem, some specific variants of the ResNet architecture proposed in \cite{chen2020image} (which we refer to as the base ResNet) are replaced with the FIS Block. We aim to have a competitive accuracy as the base ResNet while reducing the number of trainable network parameters and multiply-adds operations.

To this end, we try out four different architectures
(which we call derived or modified ResNet), namely, \emph{L23}, \emph{L2}, \emph{L3}, and \emph{Downsample}.
The L23, L2, L3, and Downsample architectures are the model architectures obtained by replacing layers 2 \& 3, layer 2, layer 3, and the corresponding base ResNet's downsample layer, respectively, with the FIS Block. Four base ResNet architectures were considered: ResNet20, ResNet32, ResNet44, and ResNet56. An illustration of ResNet20 as proposed in \cite{chen2020image} and its L23 derivative are given in Figure~\ref{fig:resnet20-architecture} and Figure~\ref{fig:resnet20-l23-architecture}, respectively. 

We emphasize two major differences between the base and derived ResNet architectures. First, each Basic Block of a base ResNet consists of two convolution layers where each convolution is followed by a Batch Normalization and RELU layers. On the other hand, an Adaptive Pooling Layer comes after the two FIS Layers in the FIS Block to pool the width and height of an input to a specific size. Second, the Basic block in a given layer of the base ResNet is repeated $n$ times based on the desired architecture. For instance, each Basic Block of the base ResNet20 is repeated $n = 3$ times; see Figure~\ref{fig:resnet20-architecture}. However, there is only one FIS Block in a given layer of the derived ResNet architecture. This is where the reduction in the number of parameters and multi-add operations stems from. 

\begin{figure}[H]
  \centering
  \includegraphics[width=0.7\textwidth]{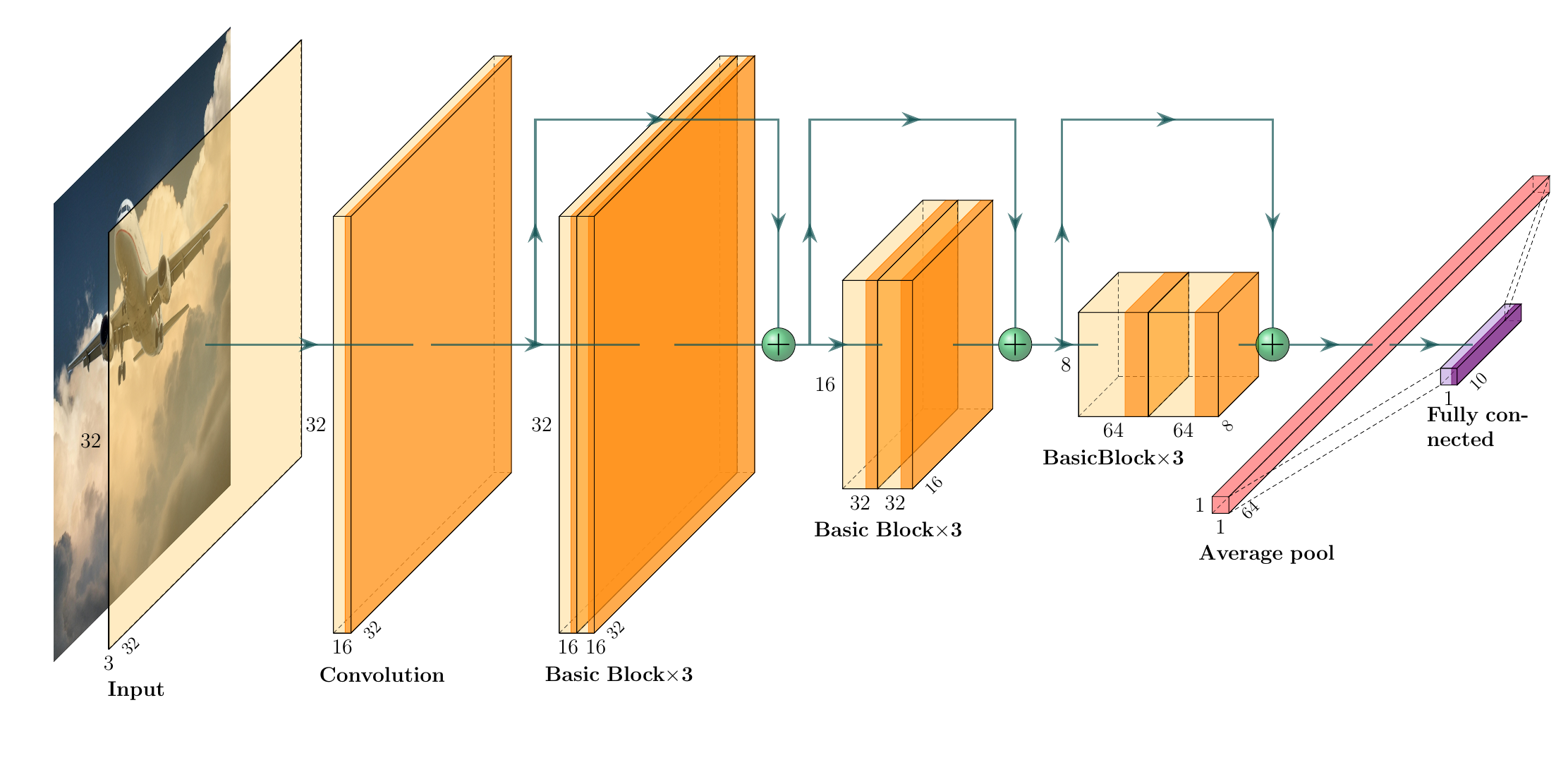}
  \caption{ResNet20 architecture designed for CIFAR-10 dataset. This is the architecture used in \cite{chen2020image}. A batch normalization and RELU layers follow each convolution layer.}
  \label{fig:resnet20-architecture}
\end{figure}

\begin{figure}[H]
  \centering
  \includegraphics[width=0.7\textwidth]{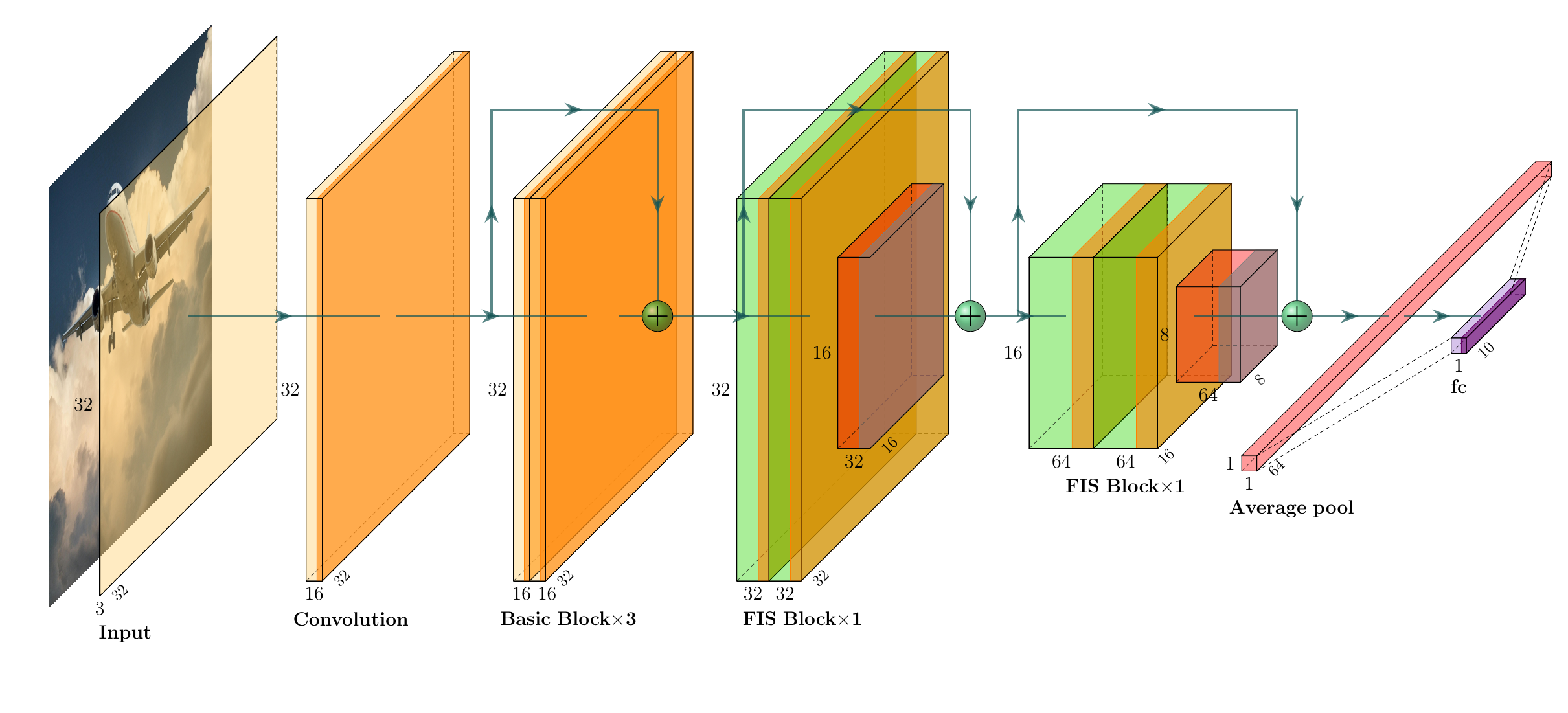}
  \caption{An exaple of modified ResNet20 architecture (L23) designed for CIFAR-10 dataset. Here, the last 2 convolution Basic Blocks are replaced by the FIS Block. A single FIS Block comprises 2 Fast Iterated Sums Layer and a Pooling Layer (which can be Max Pooling or Average Pooling). A Batch Normalization and RELU layers follow each Fast Iterated Sums Layer.}
  \label{fig:resnet20-l23-architecture}
\end{figure}

\subsubsection{Datasets and data preprocessing}
Both the base and derived ResNet architectures described in Section~\ref{sec:classification-intro} are trained and validated on two datasets (one at a time), namely, CIFAR-10 and CIFAR-100 \cite{krizhevsky2009learning}. 

The CIFAR-10 and CIFAR-100 datasets are widely used benchmarks in machine learning, each containing 60,000 color images of size 
$32 \times 32$ pixels. CIFAR-10 includes 10 classes (e.g., cat, ship, truck), while CIFAR-100 features 100 finer-grained classes grouped into 20 superclasses (e.g., animals, vehicles). Both datasets are split into 50,000 training and 10,000 test images, with CIFAR-100 offering increased complexity for more challenging image classification tasks.

Each of the datasets are preprocessed and trained using the techniques proposed in \cite{chen2020image}.
This involves the application of the composition of image transformations, including random cropping (of size 32 and padding 4), random horizontal flipping and normalization (using the mean and standard deviation of the images), to increase training data.

\subsubsection{Training setup}
The classifiers are trained using the Stochastic Gradient optimizer wrapped in a Cosine Annealing Scheduler with a learning rate of $0.1$, momentum of $0.9$, no dampening, weight decay of $5 \times 10^{-4}$, and with Nesterov option.
Except for the experiment on corner tree randomness carried out in Section~\ref{sec:classification-rand-effect}, the base and derived models are trained with 200 epochs. 

After tuning the FIS hyperparameters, we discovered that the max-plus semiring (see Section~\ref{subsec:semiring}) in combination with trees built randomly produced the best accuracy. Thus, the max-plus semiring with a random tree structure was adopted throughout the experiments unless explicitly sated otherwise (for experiments comparing different semirings and tree types, see Tables~\ref{tab:clf-perf-comp-semiring} and \ref{tab:clf-perf-comp-treetype} of the Appendix).

\subsubsection{Classification results}
Comprehensive model performance of the base and derived ResNet models on the validation set of the CIFAR-10 and CIFAR-100 are presented in Table~\ref{tab:classification-model-performance}. We emphasize two observations from this table. 

Firstly, the tiniest derived models (concerning the number of parameters and multi-add operations) under each dataset are competitive in accuracy compared to the corresponding base models. For instance, the L23 of the ResNet32 has an accuracy of $90.63\%$ and $63.96\%$ under CIFAR-10 and CIFAR-100, respectively. These accuracies differ by only $2.90\%$ and $6.20\%,$ respectively, from the corresponding base model accuracies. However, the L23 of the ResNet32 has a significant reduction in the number of parameters and multi-add operations: $\approx85\%$ and $\approx65\%$ reduction in the number of trainable parameters and multi-add operations, respectively, under both CIFAR-10 and CIFAR-100. Secondly, the accuracy of the ResNet$\{N\}$'s Downsample model is similar to the accuracy of the ResNet$\{N+12\}$'s base ($N \in \{20, 32, 44\}$) under both CIFAR-10 and CIFAR-100. For instance, under the CIFAR-10 dataset, the derived Downsample of ResNet44 has an accuracy of $94.47\%$ while the base of ResNet56 is $94.37\%$ (an accuracy difference of only $0.1\%$). On the other hand, ResNet44's Downsample has a reduction in the number of trainable parameters and multi-add operations of $\approx20\%$ and $\approx23\%$, respectively, compared to  ResNet56's base. 
\subsubsection{Ablation study}
\label{sec:classification-ablation}
To further study the effectiveness of the FIS Layer in the context of the classification task, we considered a simplified version of each ResNet's base model. This version retains the first layer of the ResNet (leaving out layers 2 and 3). We call this version \textit{Controlled Architecture (CA)}. Subsequently, we formed a derived architecture from CA by adding a single FIS Block just after the first layer; we call this \textit{CA-FIS}. Both CA and CA-FIS were trained and validated on CIFAR-10 and CIFAR-100. The results from this ablation study are presented in Table~\ref{tab:classification-ablation-study}. It was discovered that adding FIS Block to the CA architecture improved its performance across the datasets and ResNet architectures considered. For example, ResNet20's CA saw an increase of $6.57\%$ and $14.35\%$ in accuracy under CIFAR-10 and CIFAR-100, respectively, when an FIS Block was added. These results demonstrate the significance of the FIS Layer in improving the performance of the considered controlled architectures.

\subsubsection{The random effect of FIS}
\label{sec:classification-rand-effect}

As explained in \Cref{ss:FIS_layer}, the inherent randomness in a given FIS
layer stems from using different cardinal directions in building a tree (see
Section~\ref{sec:fis} for a list of all possible cardinals) as well as the
random generation of the tree structure itself.

To examine the impact of different random realizations of these structures on the FIS Layer, we employed the CA-FIS model as defined in Section \ref{sec:classification-ablation} and repeated its training 10 times on the CIFAR-10 dataset. Each training run was conducted for 20 epochs instead of the usual 200, as the focus was on assessing robustness to corner tree randomness rather than model performance. The accuracies were averaged, and the standard deviation was computed for each model. The results of this experiment are presented in Table~\ref{tab:classification-random-robustness}. We discovered that the maximum standard deviation of the top 1 and 5 accuracies are only $0.85\%$ and $0.12\%$, respectively, indicating a less significant effect of tree randomness on the FIS Layer. Thus, the FIS Layer is robust to the nature of cardinal directions used in building the corner trees.

\begin{table}[H]
    \centering
    \caption{Performance of the various proposed model architectures on the validation set of CIFAR-10 and CIFAR-100 datasets. L23, L2, L3, and Downsample are the model architectures obtained by replacing layers 2 \& 3, layer 2, layer 3, and the corresponding ResNet's downsample layer, respectively, with the Fast Iterated Sums Block (FIS Block). The Base is the ResNet architecture in \cite{chen2020image}.  Additionally, the final average pool layer of the corresponding ResNet is replaced by a composition of Fast Iterated Sums, Batch Normalization, and Pooling (which can be either average or max pool) layers. $N_p$ and $N_m$ (in Millions) are the total numbers of trainable parameters and multiply-adds, respectively.  Acc@1 and  Acc@5 (in \%) are the top 1 and 5 validation accuracies, respectively.  Numbers without and with an asterisk (*) are for the CIFAR-10 and CIFAR-100, respectively.}
    \small
    \begin{tabular}{lccccc}
        \hline
        \hline
        & & Acc@1 (\%) & Acc@5 (\%) & $N_p (M)$ & $N_m (M)$\\
        \hline\hline
       \multirow{4}{*}{ResNet20} & L23 & $89.73$, $62.75^*$ & $99.63$, $88.35^*$ & $0.06$, $0.06^*$ & $14.60$, $14.61^*$\\
        & L2 & $91.93$, $66.01^*$ & $99.76$, $89.64^*$ & $0.24$, $0.25^*$ & $27.71$, $27.71^*$\\
       & L3 & $91.60$, $65.45^*$ & $99.76$, $89.96^*$ & $0.10$, $0.11^*$ & $27.71$, $27.71^*$\\
       & Downsample & $93.35$, $68.99^*$ & $99.74$, $91.30^*$ & $0.31$, $0.32^*$ & $40.55$, $40.56^*$\\
       & Base & $92.60$, $68.83^*$ & $99.81$, $91.01^*$ & $0.27$, $0.28^*$ & $40.81$, $40.82^*$\\
       \hline 
        \multirow{4}{*}{ResNet32} & L23 & $90.63$, $63.96^*$ & $99.70$, $89.37^*$ & $0.07$, $0.07^*$ & $24.04$, $24.04^*$\\
       & L2 & $93.05$, $67.91^*$ & $99.80$, $90.50^*$ & $0.40$, $0.40^*$ & $46.58$, $46.59^*$\\
       & L3 & $92.45$, $67.46^*$ & $99.81$, $90.84^*$ & $0.15$, $0.16^*$ & $46.58$, $46.59^*$\\
       & Downsample & $93.79$, $69.93^*$ & $99.69$, $91.39^*$ & $0.51$, $0.51^*$ & $68.87$, $68.87^*$\\
       & Base & $93.53$, $70.16^*$ & $99.77$, $90.89^*$ & $0.47$, $0.47^*$ & $69.12$, $69.13^*$\\
       \hline
        \multirow{4}{*}{ResNet44} & L23 & $90.88$, $64.81^*$ & $99.69$, $89.65^*$ & $0.07$, $0.08^*$ & $33.47$, $33.48^*$\\
       & L2 & $92.91$, $68.94^*$ & $99.78$, $90.82^*$ & $0.55$, $0.56^*$ & $65.46$, $65.46^*$\\
       & L3 & $92.83$, $68.01^*$ & $99.81$, $91.13^*$ & $0.20$, $0.20^*$ & $65.46$, $65.46^*$\\
       & Downsample & $94.47$, $71.48^*$ & $99.79$, $91.88^*$ & $0.69$, $0.71^*$ & $97.18$, $97.18^*$\\
       & Base & $94.01$, $71.63^*$ & $99.77$, $91.58^*$ & $0.66$, $0.67^*$ & $97.44$, $97.44^*$\\
       \hline
        \multirow{4}{*}{ResNet56} & L23 & $90.98$, $65.20^*$ & $99.77$, $89.89^*$ & $0.09$, $0.09^*$ & $42.91$, $42.92^*$\\
       & L2 & $93.11$, $70.21^*$ & $99.77$, $90.84^*$ & $0.71$, $0.72^*$ & $84.33$, $84.34^*$\\
       & L3 & $93.28$, $69.12^*$ & $99.81$, $91.90^*$ & $0.24$, $0.25^*$ & $84.33$, $84.34^*$\\
       & Downsample & $94.28$, $72.75^*$ & $99.72$, $92.01^*$ & $0.90$, $0.90^*$ & $125.48$, $125.50^*$\\
       & Base & $94.37$, $72.63^*$ & $99.83$, $91.94^*$ & $0.86$, $0.86^*$ & $125.75$, $125.75^*$\\
       \hline
       \hline
    \end{tabular}
    \normalsize
    \label{tab:classification-model-performance}
\end{table}

\begin{table}[H]
    \centering
    \caption{Ablation study. The layer 1, average pooling and fully-connected layers of each ResNet are retained for classification and used as the ``Control Architecture (CA)". To investigate the effectiveness of the Fast Iterated Sums, a single FIS Block was added just after layer 1. These derived architectures (CA-FIS) were trained (with 200 epochs) and validated on the CIFAR-10 and CIFAR-100 datasets. $N_p$ and $N_m$ (in Millions) are the total numbers of trainable parameters and multiply-adds, respectively.  Acc@1 and  Acc@5 (in \%) are the top 1 and 5 validation accuracies, respectively. Numbers without and with an asterisk (*) are for the CIFAR-10 and CIFAR-100, respectively.}
    \small
    \begin{tabular}{lccccc}
        \hline
        \hline
        & & Acc@1 (\%) & Acc@5 (\%) & $N_p (M)$ & $N_m (M)$\\
        \hline\hline
       \multirow{2}{*}{ResNet20} & CA & $79.66$, $39.03^*$ & $99.01$, $72.2^*$ & $0.01$, $0.02^*$ & $14.60$, $14.60^*$\\
        & CA-FIS & $86.23$, $53.38^*$ & $99.43$, $83.23^*$ & $0.02$, $0.03^*$  & $14.60$, $14.60^*$\\
       \hline 
        \multirow{2}{*}{ResNet32} & CA & $83.33$, $44.07^*$ & $99.32$, $76.25^*$ & $0.02$, $0.03^*$ & $24.04$, $24.04^*$\\
       & CA-FIS & $87.74$, $59.98^*$ & $99.61$, $84.46^*$ & $0.03$, $0.04^*$ & $24.04$, $24.04^*$\\
       \hline
        \multirow{2}{*}{ResNet44} & CA & $84.81$, $47.68^*$ & $99.36$, $79.71^*$ & $0.03$, $0.03^*$ & $33.47$, $33.47^*$\\
       & CA-FIS  & $88.71$, $57.03^*$ & $99.69$, $85.04^*$ & $0.04$, $0.05^*$ & $33.47$, $33.48^*$\\
       \hline
        \multirow{2}{*}{ResNet56} & CA & $85.96$, $49.32^*$ & $99.57$, $80.87^*$ & $0.04$, $0.04^*$ & $42.91$, $42.91^*$\\
       & CA-FIS & $88.72$, $57.55^*$ & $99.59$, $85.96^*$ & $0.05$, $0.06^*$ & $42.91$, $42.91^*$\\
       \hline
       \hline
    \end{tabular}
    \normalsize
    \label{tab:classification-ablation-study}
\end{table}

\begin{table}[H]
    \centering
    \caption{Robustness of the FIS Layer to Corner Tree randomness. The CA-FIS model defined in Section~\ref{sec:classification-rand-effect} was repeatedly trained 10 times on the CIFAR-10 dataset using 20 epochs. The average and standard deviation of accuracies are reported.}
    \small
    \begin{tabular}{lcccc}
      \hline
      \hline
      & \multicolumn{2}{c}{Acc@1 (\%)} &  \multicolumn{2}{c}{Acc@5 (\%)}\\
     \cline{2-3} \cline{4-5} 
     & $\bar{x}$ & $\sigma$ & $\bar{x}$ & $\sigma$\\
     \hline
      ResNet20 & 74.23 & 0.57 & 98.38 & 0.11\\ 
      ResNet32 & 76.18 & 0.85 & 98.61 & 0.12\\ 
      ResNet44 & 77.42 & 0.50 & 98.73 & 0.08\\ 
      ResNet56 & 78.31 & 0.83 & 98.84 & 0.12\\ 
      \hline
      \hline
    \end{tabular}
    \normalsize
    \label{tab:classification-random-robustness}
\end{table}

\subsection{Anomaly detection}
\label{ss:anomaly-detection}

Fast iterated sums can be seen as a generalized version of the 2D-signatures used in \cite{ZLT_2022} and \cite{xie20252dsigdetectsemisupervisedframeworkanomaly}
where they have shown potential in texture classification and anomaly detection respectively. We consider anomaly detection as our second application, though under a different setup from that explored in  \cite{xie20252dsigdetectsemisupervisedframeworkanomaly}. As anomaly detection for image-related tasks requires texture-level information \cite{armi2019texture}, leveraging fast iterated sums provides a structured approach to capturing fine-grained patterns and distinguishing normal from anomalous regions. This is particularly relevant in applications such as medical imaging, defect detection in manufacturing, and remote sensing, where anomalies often manifest as subtle texture deviations \cite{tschuchnig2021anomaly}.

\subsubsection{Dataset and data preprocessing}
For our anomaly detection task, we use the MVTec AD dataset \cite{bergmann2019mvtec}, which consists of high-resolution texture images commonly used for industrial defect detection. Those texture images are utilized in \cite{aota2023zero}, ensuring consistency with prior work on texture-based anomaly detection. The datasets contain five categories of high-resolution textures, namely carpet, grid, leather, tile, and wood. Each category has training and test sets.

The training set contains normal images only, while the test set consist of both normal and anomaly images. In addition, each category's test set has different kinds of faults. For instance, carpet has five fault groups: \textit{color, cut, hole, metal contamination, and thread}. This makes the data a challenging task for any anomaly detection algorithm. In total, there are 1266 training images and 515 test images; see Table~\ref{tab:ad-data-brkdwn} for the breakdown of the number of images by category.
\begin{table}[H]
    \centering
        \caption{The breakdown of the texture images from the MVTec AD \cite{bergmann2019mvtec} dataset. The training set contains only normal images, while the test set consists of both normal and anomaly images.
        }
    \small
    \begin{tabular}{lcc} 
    \hline
    \hline
   \multirow{2}{*}{Category} & \multicolumn{2}{c}{Number of images}\\
   \cline{2-3}
   & Train & Test
    \\
    \hline 
     Carpet & 270 & 117\\
     Grid & 264 & 78 \\
     Leather & 245 & 124\\ 
     Tile & 230 & 124\\
     Wood & 247 & 79\\
    \hline
    \textbf{Total} 
    & 1266 & 515 \\
    \hline
    \hline
    \end{tabular}
    \normalsize
    \label{tab:ad-data-brkdwn}
\end{table}
Following the convention used in \cite{roth2022towards}, each image is resized to $224 \times 224$ and normalized using the ImageNet mean and standard deviation ($\bar{x}=(0.485, 0.456, 0.406)$, $\sigma = (0.229, 0.224, 0.225)$). The training set was further split into training and validation sets using a $90\%$-$10\%$ train-validation split.
\subsubsection{Autoencoder network architecture}
\label{sec:ae-network-architecture}

This section explains how FIS Layers function as an encoder in an autoencoder network for anomaly detection. The use of autoencoder for anomaly detection has been explored in \cite{davletshina2020unsupervised, sakurada2014anomaly, huang2019inverse}. We emphasize \cite{huang2019inverse} as they used the MVTec Dataset as one of the benchmarks. In their research, a novel anomaly detection method is introduced by reformulating reconstruction-based approaches into a restoration task, where selected attributes are erased and then recovered to enforce the learning of semantic feature embeddings. During testing, it is expected that anomalies will exhibit high restoration errors since the model learns attributes only from normal data.

We consider the FIS layer as building blocks for an encoder in an autoencoder architecture to directly measure how well it captures the most important features of the input data. The proposed autoencoder's encoder consists of two FIS Layers (with each followed by RELU and Batch Normalization layers). The decoder part is a composition of two convolution layers, with the first layer followed by RELU and Batch Normalization layers. Figure~\ref{fig:ae-architecture} illustrates the proposed autoencoder.

Inputs to the autoencoder are the features extracted from the original images using an ImgaNet-pretrained ResNet. The extracted features are obtained by concatenating the outputs from layer 2 and layer 3 of the pre-trained model, with the model’s weights remaining fixed throughout the feature extraction process.

Since the input dimensions have already been reduced to a manageable size (in our case, $28 \times 28$), the autoencoder operates solely on the channel dimension. For all the pre-trained models considered in this study, the concatenated feature representation consists of $1536$ channels. Consequently, the autoencoder is parameterized by the latent dimension, which determines the number of channels retained in the latent space (this is the $l$ in Figure~\ref{fig:ae-architecture}). We name the proposed autoencoder described in Figure~\ref{fig:ae-architecture} as \textit{FIS-AE}.

\begin{figure}[H]
  \centering
  \includegraphics[width=0.7\textwidth]{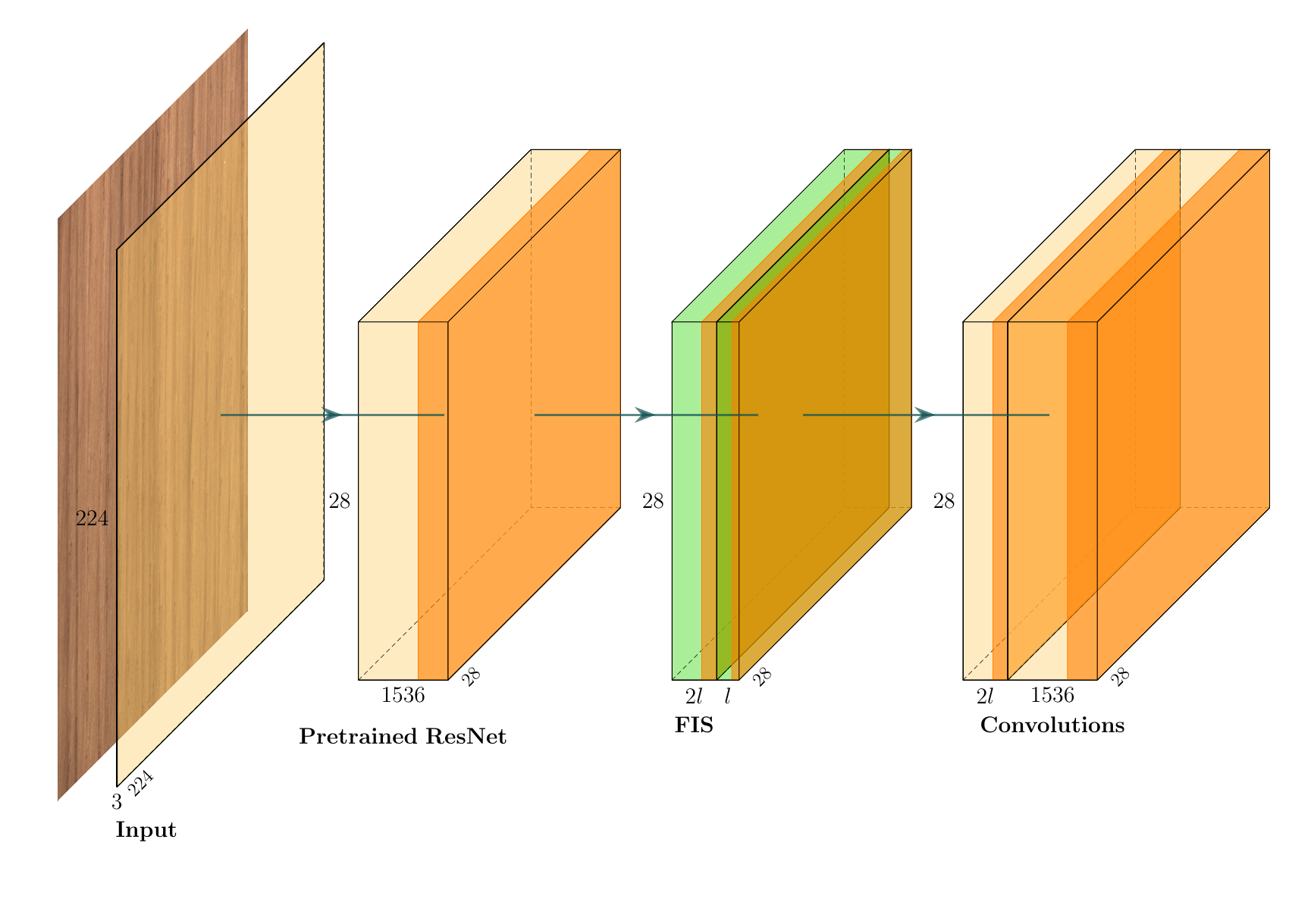}
  \caption{FIS-AE architecture. Features are extracted from the normal images using the ImageNet pre-trained ResNet, which are then fed to the autoencoder. These features are encoded by the FIS Layers using $l$ channels.}
  \label{fig:ae-architecture}
\end{figure}

\subsubsection{Autoencoder as anomaly detector}
\label{sec:ae-as-ad}
The way FIS-AE described in Section~\ref{sec:ae-network-architecture} was used for anomaly detection is explained as follows. FIS-AE is trained exclusively on normal images to reconstruct the input features extracted from an ImageNet-pretrained ResNet. The reconstruction error, measured as the mean squared error in our case, is expected to be low for normal images and significantly higher for anomalous images \cite{huang2019inverse}.

The reconstruction errors serve as the basis for estimating anomaly scores, which can then be used to determine anomaly thresholds and evaluate various metrics, such as the area under the receiver operating characteristic curve (AUROC), precision, and recall. Specifically, given a trained FIS-AE and a set of test images (comprising both normal and anomalous samples), the reconstruction error for each image is computed as the mean squared error over the channel dimension of the extracted features.
 Following the description in Section~\ref{sec:ae-network-architecture}, the error on an image will have the shape of $28 \times 28$. The error is then flattened and sorted in descending order. Consequently, following a similar convention in
 \cite{roth2022towards}, the anomaly score for the corresponding image is calculated by finding the mean of the first 10 highest flattened error values (see Table~\ref{tab:ad-topn-robustness} of the Appendix for the results of using different thresholds other than 10). This approach is considered to have a robust anomaly score. Various metrics can be calculated once the anomaly scores are computed for all the images in the test set.

\subsubsection{Training setup}
\label{sec:ad-data-train}

 FIS-AE was trained on one category at a time using the mean squared error criterion and Adam optimizer with a learning rate of $0.001$. We set the default number of epochs to $200$. As discussed in Section~\ref{sec:ae-network-architecture}, FIS-AE requires an ImageNet-pretrained ResNet for feature extraction. To this end, we set ImageNet-pretrained Wide-ResNet50 as the default backbone. In addition, $l=32$ was set as the default for the latent dimension (the number of channels to keep from the extracted features); see Tables~\ref{tab:ad-latent-dim-robustness} and \ref{tab:ad-backbone-robustness} for the results of using different values of $l$ and ImageNet-pretrained ResNets respectively in the Appendix. Following the classification task experiments, max-plus semiring with a random tree structure was adopted.

\subsubsection{Test metrics}
\label{sec:ad-data-train}
The performance of FIS-AE was assessed on the test set using image-level AUROC as set out in the last paragraph of Section~\ref{sec:ae-as-ad}. Concretely, the AUCROC was calculated for each data category using the \texttt{roc_auc_score} from the \texttt{metrics} module of scikit-learn: where \texttt{y_true} is set to an array of 0's and 1's (with 0 and 1 depicting normal and anomaly images, respectively); and \texttt{y\_score} set to the array of the calculated anomaly scores.

\subsubsection{Anomaly detection results and discussion}
The performance of FIS-AE (with the default configurations as set out in Secation~\ref{sec:ad-data-train}) on the test set of each category is presented in Table~\ref{tab:ad-model-performance}. Overall, the highest AUROC score ($100\%$) was attained on the leather category, whereas the least ($89.6\%$) was observed on the grid. The rationale behind the lowest performance on the grid can be associated to the nature of the corresponding fault types in this category. Faults in the grid are less pronounced and are, in many cases, similar to normal images. Thus, the autoencoder will be tricked into recreating fault features as normal ones and vice versa.

We compare our results to those in the literature. In particular, studies in \cite{huang2019inverse, aota2023zero, roth2022towards} were chosen for comparison. The study in \cite{aota2023zero} was chosen as they also considered anomaly detection in texture images, while \cite{huang2019inverse} was considered as they introduced an encoding method. Furthermore, \cite{roth2022towards} was included because features extracted by ImageNet-pretrained ResNet were used in their study. Overall, our proposed FIS-AE has competitive AUROC scores. For instance, FIS-AE is much better than that of \cite{huang2019inverse} in all categories and slightly better than \cite{roth2022towards} under the carpet category ($1.2\%$ higher) and \cite{aota2023zero} under the wood category ($0.4\%$).

\begin{table}[H]
    \centering
    \caption{Image-level anomaly detection performance on MVTec AD \cite{bergmann2019mvtec} dataset.  The model (that uses the pre-trained Wide-ResNet50 as the backbone) with the best average AUROC was chosen from the PatchCore \cite{roth2022towards} paper for comparison.}
    \small
    \begin{tabular}{lccccc}
      \hline
      \hline
      & \multicolumn{5}{c}{AUROC (\%)}\\
       \cline{2-6}
      & Carpet & Grid & Leather & Tile & Wood\\
     \hline
     ITAE \cite{huang2019inverse} & 70.6 & 88.3 & 86.2 & 73.5 & 92.3\\
      PatchCore-25 \cite{roth2022towards} & 98.7 & 98.2 & 100  & 98.7 & 99.2\\
      Zero-shot \cite{aota2023zero} & 99.9 & 100 & 100 & 99.1 & 98.9\\
      \hline
      FIS-AE & 99.9 & 89.6 & 100 & 97.7 & 99.3\\
      \hline
      \hline
    \end{tabular}
    \normalsize
    \label{tab:ad-model-performance}
\end{table}

\subsubsection{Ablation study on FIS-AE}
To further demonstrate the effectiveness of using FIS layers in the encoder network, we carried out an ablation study on FIS-AE architecture. To this end, we replaced the FIS layers in the autoencoder accordingly. The results from this experiment are presented in Table~\ref{tab:ad-ablation-study}.

Overall, the average AUROC score (calculated over all the categories) is higher ($1.2\%$ more) when the FIS layers are used in the encoder network. In addition, using FIS layers produced a lower standard deviation of $3.9\%$ AUROC compared to $6.3\%$ when convolution layers are used in the encoder network. In particular, there is a significant difference in the performance on the grid category, where the FIS-backed encoder is $6\%$ AUROC higher than that of the convolution-backed encoder.
\begin{table}[H]
    \centering
    \caption{Ablation study. Image-level anomaly detection performance on MVTec AD \cite{bergmann2019mvtec} dataset. We use pure convolution layers in the encoder as a substitute for FIS layers.}
    \small
    \begin{tabular}{lccccccc}
      \hline
      \hline
      & \multicolumn{5}{c}{AUROC (\%)}\\
       \cline{2-8}
      & Carpet & Grid & Leather & Tile & Wood & $\bar{x}$ & $\sigma$\\
     \hline
       Ablation & 99.9 & 83.6 & 100 & 97.8 & 99.4 & 96.1 & 6.3\\
       FIS-AE & 99.9 & 89.6 & 100 & 97.7 & 99.3 & 97.3 & 3.9\\
      \hline
       \hline
    \end{tabular}
    \normalsize
    \label{tab:ad-ablation-study}
\end{table}

\appendix
\section{Appendix}

We collect here some additional results and figures referenced in the main text. 
\label{sec:appendix}

\begin{table}[H]
    \centering
    \caption{Classification model performance comparison using the two semirings: real and max-plus. The CA-FIS model defined in Section~\ref{sec:classification-rand-effect} was trained and validated on the CIFAR-10 dataset for 200 epochs. The validation accuracies are reported.}
    \small
    \begin{tabular}{lcccc}
      \hline
      \hline
      & \multicolumn{2}{c}{Acc@1 (\%)} &  \multicolumn{2}{c}{Acc@5 (\%)}\\
     \cline{2-3} \cline{4-5} 
     & real & max-plus &  real & max-plus\\
     \hline
      ResNet20 & 85.57 & 86.23 & 99.37 & 99.43\\ 
      ResNet32 & 86.71 & 87.74 & 99.53 & 99.61\\ 
      ResNet44 & 87.50 & 88.71 & 99.53 & 99.69\\
      ResNet56 & 88.15 & 88.72 & 99.60 &  99.59\\ 
      \hline
      \hline
    \end{tabular}
    \normalsize
    \label{tab:clf-perf-comp-semiring}
\end{table}

\begin{table}[H]
    \centering
    \caption{Classification model performance comparison using the three proposed tree types: random, linear and linear-NE. The CA-FIS model defined in Section~\ref{sec:classification-rand-effect} was trained and validated on the CIFAR-10 dataset for 200 epochs. The validation accuracies are reported.}
    \small
    \begin{tabular}{lcccccc}
      \hline
      \hline
      & \multicolumn{3}{c}{Acc@1 (\%)} &  \multicolumn{3}{c}{Acc@5 (\%)}\\
     \cline{2-4} \cline{5-7} 
     & random & linear &  linear-NE & random & linear &  linear-NE\\
           \hline
      ResNet20 & 86.23 & 85.63 & 76.01 & 99.43 & 99.45 & 98.66\\ 
      ResNet32 & 87.74 & 86.83 & 68.28 & 99.61 & 99.53 & 97.37\\ 
      ResNet44 & 88.71 & 87.68 & 84.14 & 99.69 & 99.54  & 99.34\\
      ResNet56 & 88.72 & 88.26 & 85.35 & 99.59 & 99.54 & 99.37\\ 
      \hline
      \hline
    \end{tabular}
    \normalsize
    \label{tab:clf-perf-comp-treetype}
\end{table}

\begin{table}[H]
    \centering
    \caption{The effect of using different top $n$ error values (sorted in descending order) in the flattened array of reconstruction errors (of length $28^2$) on the performance of the proposed anomaly detection model. As explained in Section~\ref{sec:ae-as-ad}, the mean of the first $n$ values in the sorted array was used as the anomaly score. The mean and standard AUROC scores across the datasets are reported. For this experiment, ImageNet-pretrained Wide-ResNet50 was used and latent dimension ($l$) was fixed to 32.}
    \small
    \begin{tabular}{lcc}
      \hline
      \hline
        \multirow{2}{*}{Top $n$} & \multicolumn{2}{c}{AUROC (\%)}\\
       \cline{2-3}
      & $\bar{x}$ & $\sigma$\\
     \hline
      5 & 96.8 & 5.0\\
      10 & 97.3 & 4.4 \\
      15 & 97.5 & 4.2\\
      20 & 97.6 & 3.9\\
      25 & 97.7 & 3.8\\
      30 & 97.8 & 3.7\\
      $28^2$ & 92.6 & 10.0\\
      \hline
       \hline
    \end{tabular}
    \normalsize
    \label{tab:ad-topn-robustness}
\end{table}

\begin{table}[H]
    \centering
    \caption{The effect of latent dimension on anomaly detection performance on MVTec AD \cite{bergmann2019mvtec} dataset. The mean and standard AUROC scores across the datasets are reported. For this experiment, ImageNet-pretrained Wide-ResNet50 was used.}
    \small
    \begin{tabular}{ccc}
      \hline
      \hline
        \multirow{2}{*}{Latent dimension} & \multicolumn{2}{c}{AUROC (\%)}\\
       \cline{2-3}
       & $\bar{x}$ & $\sigma$\\
     \hline
      4 & 88.4 & 11.7\\
      8 & 93.6 & 8.1\\
      16 & 96.5 & 5.0\\
      32 & 97.4 & 3.8\\
      \hline
       \hline
    \end{tabular}
    \normalsize
    \label{tab:ad-latent-dim-robustness}
\end{table}

\begin{table}[H]
    \centering
    \caption{The effect of using different pre-trained ResNet as the auto-encoder backbone on anomaly detection performance on MVTec AD \cite{bergmann2019mvtec} dataset. The mean and standard AUROC scores across the datasets are reported. For this experiment, the latent dimension ($l$) was fixed to 32.}
    \small
    \begin{tabular}{lcc}
      \hline
      \hline
        \multirow{2}{*}{Pre-trained model} & \multicolumn{2}{c}{AUROC (\%)}\\
       \cline{2-3}
      & $\bar{x}$ & $\sigma$\\
     \hline
      ResNet50 & 96.2 & 4.7\\
      Wide-ResNet50 & 97.3 & 3.1\\
      ResNet101 & 95.1 & 8.2\\
      Wide-ResNet101 & 95.7 & 5.2\\
      \hline
       \hline
    \end{tabular}
    \normalsize
    \label{tab:ad-backbone-robustness}
\end{table}

\section*{Funding}
R.~Ibraheem is a PhD student in EPSRC's MAC-MIGS Centre for Doctoral Training. MAC-MIGS is supported by the UK's Engineering and Physical Science Research Council (grant number EP/S023291/1).

J.~Diehl is partially supported by 
DFG grant 539875438 ``Counting permutation and chirotope patterns: algorithms, algebra, and applications''
through the DFG priority programme SPP 2458 "Combinatorial Synergies".
The project is co-financed by the European Regional Development Fund (ERDF) within the framework of the Interreg VIA programme.

L. Schmitz acknowledges support from DFG CRC/TRR 388 ``Rough
Analysis, Stochastic Dynamics and Related Fields''. 

\section*{Competing interest declaration}
The authors declare that they have no known competing financial interests or personal relationships that could have appeared to influence the work reported in this paper.

\section*{Data availability}
The data used for all the experiments and modelling in this research are publicly available CIFAR10, CIFAR100 and MVTec AD datasets.

\section*{Code availability}
The implementation of our proposed methods, as well as the experiments reported,
are publicly available at \url{https://github.com/diehlj/fast-iterated-sums}.


\end{document}